\documentclass[UTF8]{article}

\usepackage[utf8]{inputenc} 
\usepackage[T1]{fontenc}    
\usepackage{nicefrac}       
\usepackage{xcolor}         

\usepackage{microtype}
\usepackage{graphicx}
\usepackage{subcaption}
\usepackage{multirow}
\usepackage{booktabs} 
\usepackage{diagbox}

\usepackage[letterpaper]{geometry}
\usepackage[hidelinks]{hyperref}
\usepackage{url}

\usepackage{amssymb}
\usepackage{mathtools}
\usepackage{amsthm}
\usepackage{array}
\usepackage{bm}
\allowdisplaybreaks

\usepackage[ruled, linesnumbered, boxed, vlined]{algorithm2e}

\usepackage{natbib}

\usepackage{subfiles}
\theoremstyle{plain}
\newtheorem{theorem}{Theorem}
\newtheorem{theorem*}{Theorem}
\newtheorem{proposition}{Proposition}
\newtheorem{lemma}{Lemma}
\newtheorem*{lemma*}{Lemma}
\newtheorem{corollary}{Corollary}
\theoremstyle{definition}

\newtheorem{assumption}{Assumption}
\theoremstyle{remark}
\newtheorem{remark}{Remark}

\DeclareMathOperator{\diag}{\mathrm{diag}}
\DeclareMathOperator{\tr}{\mathrm{tr}}

\newcommand{\rme}{\mathord{\mathrm{e}}}
\newcommand{\e}{{\bm{e}}}
\renewcommand{\tilde}{\widetilde}
\renewcommand{\hat}{\widehat}

\newcommand{\x}{{\bm{x}}}
\newcommand{\y}{{\bm{y}}}

\newcommand{\z}{{\bm{z}}}
\newcommand{\w}{{\bm{w}}}
\newcommand{\p}{{\bm{p}}}
\newcommand{\bu}{{\bm{u}}}
\newcommand{\bv}{{\bm{v}}}
\newcommand{\bvt}{\tilde{\bm{v}}}
\newcommand{\xs}{{\bm{x}}_\star}
\newcommand{\ws}{{\bm{w}}_\star}
\newcommand{\wh}{\hat{{\bm{w}}}}

\newcommand{\sig}{{\bm{\sigma}}}

\newcommand{\A}{\mathcal{A}}

\newcommand{\R}{\mathbb{R}}
\newcommand{\E}{\mathbb{E}}

\newcommand{\I}{\mathbb{I}}

\newcommand{\N}{\mathbb{N}}
\newcommand{\D}{\mathcal{D}}
\newcommand{\C}{\mathcal{C}}

\newcommand{\W}{\mathcal{W}}
\newcommand{\F}{\mathcal{F}}
\newcommand{\cP}{\mathcal{P}}
\newcommand{\Q}{\mathcal{Q}}
\renewcommand{\O}{\mathcal{O}}
\newcommand{\Ot}{\tilde{\mathcal{O}}}

\newcommand{\argmax}{\mathop{\mathrm{argmax}}}
\newcommand{\argmin}{\mathop{\mathrm{argmin}}}

\newcommand{\reg}{\mathop{\mathrm{Regret}}}
\newcommand{\pr}[1]{\mathop{\mathrm{Pr}} \left[ #1 \right]}
\newcommand{\Max}[1]{\max \left\{ #1 \right\}}

\usepackage{xparse}
\NewDocumentCommand{\Exp}{mg}{\IfNoValueTF{#2}{\E\left[#1\right]}{\E_{#2}\left[#1\right]}}

\NewDocumentCommand{\tail}{mg}{\IfNoValueTF{#2}{\Delta_{#1}}{\Delta_{#1}^{#2}}}

\newcommand{\Ws}{W_\star}

\newcommand{\Wv}{\overrightarrow{W}}
\newcommand{\Wvs}{\overrightarrow{W}_\star}
\newcommand{\Ht}{\tilde{H}}
\newcommand{\Zt}{\tilde{Z}}

\newcommand{\Ut}{\tilde{U}}
\newcommand{\Vt}{\tilde{V}}
\newcommand{\Lamt}{\tilde{\Lambda}}
\newcommand{\Sigt}{\tilde{\Sigma}}

\newcommand{\rwdv}{{\bm{r}}}
\newcommand{\rwd}{r}

\DeclareMathOperator{\smooth}{\mathrm{smooth}}
\newcommand{\dfst}{\varepsilon^{\mathtt{fst}}}
\newcommand{\dsnd}{\varepsilon^{\mathtt{snd}}}
\newcommand{\ellt}{\tilde{\ell}}
\newcommand{\ellh}{\hat{\ell}}

\title{\textbf{Efficient Multinomial Logistic Bandit via\\ Frequent Directions}}

\author{
\begin{minipage}{\textwidth}
\centering
	Linzhe He$^{1,2}$ \quad Yu-Jie Zhang$^{3}$ \quad Sifan Yang$^{1,2}$ \quad Lijun Zhang$^{1,2}$\\[0.3em]
	$^{1}$State Key Laboratory of Novel Software Technology, Nanjing University, Nanjing 210023, China\\
	$^{2}$School of Artificial Intelligence, Nanjing University, Nanjing 210023, China\\
	$^{3}$Paul G. Allen School of Computer Science \& Engineering, University of Washington, USA\\[0.3em]
	\texttt{helz@lamda.nju.edu.cn} \quad \texttt{yujiez7@cs.washington.edu}\\
	\texttt{yangsf@lamda.nju.edu.cn} \quad \texttt{zlj@nju.edu.cn}
\end{minipage}
}
\date{}

\begin{document}
\maketitle

\begin{abstract}
This paper studies efficient online algorithms for multinomial logistic bandits (MLogB),
where the feedback distribution over $K+1$ outcomes follows a multinomial logistic model of $d$-dimensional action vectors.
A representative UCB-type algorithm, OFUL-MLogB, achieves a regret bound of $\Ot(Kd\sqrt{T})$,
but still requires $\O(K^3d^3)$ time and $\O(K^2d^2)$ space per round due to parameter estimation and optimistic reward construction,
which is prohibitive in high-dimensional settings.
To address this limitation, we propose EOFD-MLogB, which integrates frequent directions matrix sketching into OFUL-MLogB.
By maintaining a low-rank SVD sketch of the accumulated Hessian,
constrained online Newton updates in parameter estimation and $Kd \times K$ spectral-norm computations in the reward bonus
are reduced to one-dimensional root-finding tasks and $K \times K$ eigenvalue computations, respectively.
This yields dominant per-round time complexity $\O(Kd(m+K)^2)$ and space complexity $\O(Kd(m+K))$,
where $m \ll d$ is the sketch size.
We further prove a regret bound of $\Ot(\tail{T}(Kd\ln\tail{T}+m)\sqrt{T})$,
where the sketching error factor $\tail{T}$ is controlled by the $m$-truncated spectral tail of the Hessian.
Thus, when the Hessian is approximately low-rank, the regret is close to that of OFUL-MLogB.
Experiments validate the computational efficiency and competitive performance.
\end{abstract}

\section{Introduction}

Stochastic multinomial logistic bandits (MLogB) \citep{amani2021ucb-based-MNL} model a class of sequential decision-making problems with categorical feedback.
At each round $t\in\N_+$, an agent selects an action $\x_t\in\A_t\subseteq\R^d$ and observes environment feedback $y_t\in[K]\cup\{0\}$,
where each outcome is associated with a reward $\rwd_k \in \R_+$.
The feedback distribution is parameterized by an unknown environment parameter $W_* \in \R^{K\times d}$ through $\pr{y_t=k\mid \x_t}=[\sig(W_*\x_t)]_k$, where $\sig(\cdot)$ is the multinomial logistic function:
\begin{equation}
	[\sig(\z)]_0 = \frac{1}{1+\sum_{k=1}^K \exp(z_k)},
	\quad
	[\sig(\z)]_i = \frac{\exp(z_i)}{1+\sum_{k=1}^K \exp(z_k)}, \quad i\ge 1 .
\end{equation}
This framework captures a wide range of applications such as recommendation systems \citep{li2010recommandation,li2017provablyGLB},
where different user responses to a recommended item are mapped to different rewards.
The goal is to minimize the cumulative regret
\begin{equation}
	\reg(T)
	=
	\sum_{t=1}^{T}
	\max_{\x\in\A_t} \rwdv^\top \sig(W_* \x)
	-
	\rwdv^\top \sig(W_* \x_t).
\end{equation}

A standard approach is the upper confidence bound (UCB),
which maintains confidence sets for $W_*$ and selects actions using optimistic rewards.
A prior UCB-type algorithm, MNL-UCB \citep{amani2021ucb-based-MNL}, combines batch maximum likelihood estimation with a first-order Taylor-based optimistic reward approximation,
achieving provable regret guarantees.
Its improved variant \citep[Appendix C]{amani2021ucb-based-MNL} further employs second-order Taylor-based optimistic reward approximation to construct a tighter optimistic rewards for the non-convex multinomial logistic reward,
but still relies on batch parameter estimation and hence full-history storage.
OFUL-MLogB \citep{zhang2023mlb} avoids this storage cost by using online Newton step (ONS) \citep{Hazan2007ons} for parameter estimation,
while retaining the second-order optimistic reward design.
However, it still maintains and manipulates a $Kd\times Kd$ Hessian matrix,
leading to per-round time complexity $\O(|\A_t|K^3d^2+K^3d^3)$ and space complexity $\O(K^2d^2)$,
which are prohibitive in high-dimensional settings.

The bottleneck comes from the role of a $Kd \times Kd$ accumulated Hessian in both parameter estimation and action selection.
The ONS parameter update requires solving a quadratic-constrained quadratic program (QCQP),
while the optimistic reward construction involves high-dimensional spectral norm evaluations;
both scale poorly with the Hessian dimension.
Although matrix sketching has been successfully used to accelerate stochastic linear bandits \citep{kuzborskij2019efficient}, the nonlinear multinomial logistic model prevents analogous closed-form reductions in the associated QCQP and spectral norm computations.
Consequently, existing sketching techniques do not directly remove these MLogB bottlenecks.
Similar Hessian-induced bottlenecks also arise in ONS-based generalized linear bandits \citep{zhang2026GLBonepass} and other second-order MLogB algorithms \citep{Lee2024r2cs,boudart2025enjoyNonLinear}.

We propose EOFD-MLogB, an efficient variant of OFUL-MLogB based on frequent directions (FD) matrix sketching \citep{ghashami2016fd}.
By maintaining a rank-$m$ SVD sketch of the accumulated Hessian, EOFD-MLogB reduces the space complexity from $\O(K^2d^2)$ to $\O(Kd(m+K))$.
Moreover, the sketched spectral structure reduces the $Kd$-dimensional QCQP in ONS to a one-dimensional root-finding problem,
and reduces the $Kd\times K$ spectral norm computations in optimistic reward construction to $K\times K$ computations.
As a result, the per-round computational cost becomes $\O(|\A_t|(K^3+mK^2+dmK)+dK(m+K)^2)$.

On the theoretical side, we prove a regret bound $\Ot(\tail{T}(Kd\ln\tail{T}+m)\sqrt{T})$, where $\tail{T} = 1 + \rho_{1:T} / \lambda$ is the sketching factor, and $\rho_{1:T}$ is the sketching error term bounded by the $m$-truncated spectral tail of the Hessian.
Thus, when the accumulated Hessian is approximately low-rank, EOFD-MLogB achieves regret close to the $\Ot(Kd\sqrt{T})$ regret of OFUL-MLogB while substantially improving computational efficiency.
Finally, experiments on \texttt{MNIST}, synthetic data, and additional real datasets demonstrate significant speedups with competitive regret performance.

\paragraph{Notations.}
For $W\in\R^{K\times d}$, let $\Wv\in\R^{Kd}$ denote its row-wise vectorization.
We use $\|\cdot\|_F$ and $\|\cdot\|_2$ for Frobenius and spectral norms, respectively,
$\otimes$ for the Kronecker product, and $\|\x\|_Z=\sqrt{\x^\top Z\x}$ for the norm induced by a positive semidefinite matrix $Z$.
We use $\O(\cdot)$ to hide constants, highlighting the dependency on $m,d,K,\kappa,T$,
and $\Ot(\cdot)$ to further hide polylogarithmic factors.


\section{Related Work}

In this section, we review prior work related to MLogB and frequent directions.

\subsection{MLogB and OFUL-MLogB} \label{ssec:logistic-bandit-setting}
MLogB extends logistic bandits ($K=1$), a canonical instance of generalized linear bandits (GLB) \citep{filippi2010glb,li2017provablyGLB,ding2021sgdts4glb}.
For logistic bandits, a UCB-type algorithm OL$^2$M \citep{zhang2016logistic} employed ONS to estimate the parameter;
subsequent works \citep{abeille2021minimax,faury2022jointly-logistic-bandit} improved the regret bound to the optimal order.
The multinomial case ($K>1$) is more challenging, because maximizing the expected multinomial logistic reward is generally a non-convex problem.
The improved variant of MNL-UCB \citep[Appendix C]{amani2021ucb-based-MNL} and subsequent UCB-type algorithms \citep{Lee2024r2cs} handle this non-convexity using second-order optimistic reward approximations,
but rely on batch MLE and storage of the full interaction history.
The state-of-the-art ONS-based OFUL-MLogB algorithm \citep{zhang2023mlb} combines second-order Taylor expansion of optimistic rewards with curvature-aware online Newton step (ONS) updates,
eliminating full-history storage but requiring the maintenance and manipulation of a $Kd\times Kd$ Hessian matrix.

We next summarize the components of OFUL-MLogB that are relevant to our method.
After selecting $\x_t$ and observing $y_t$, we define the logistic loss $\ell_t(W)=-\ln[\sig(W\x_t)]_{y_t}$.
OFUL-MLogB maintains an estimate $W_t$ of $W_*$, updated by a proximal ONS step each round:
\begin{equation} \label{eq:oful-mlogb-W-update}
	W_{t+1}
	=
	\argmin_{W\in\W}
	\underbrace{
		\left< \nabla \ell_t(W_t), \Wv - \Wv_t \right>
		+
		\frac{1}{2}
		\left\| \Wv - \Wv_t \right\|_{\nabla^2 \ell_t(W_t)}^2
	}_{\ellt_t(W)}
	+
	\frac{1}{2\eta}
	\left\| \Wv - \Wv_t \right\|_{H_t}^2 ,
\end{equation}
where $\ellt_t(W)$ is a second-order surrogate of $\ell_t$ at $W_t$ up to an additive constant.
The Hessian is updated as
$H_{t+1}=H_t+\nabla^2\ell_t(W_{t+1})$, with $H_1=\lambda I_{Kd}$.

For action selection, OFUL-MLogB chooses $\x_t$ according to the UCB principle under probability at least $1-\delta$:
\begin{equation} \label{eq:oful-mlogb-decision-rule}
	\x_t = \argmax_{\x \in \A_t} R_t(\x,\delta),
\end{equation}
where
\begin{equation} \label{eq:oful-mlogb-optimistic-reward}
	R_t(\x,\delta)
	=
	\begin{cases}
		\max_{\w \in \C_t} \sigma(\w^\top \x), & K = 1 \\
		\rwdv^\top \sig(W_t \x) + b_t(\x,\delta), & K \ge 2
	\end{cases}.
\end{equation}
Here $\C_t=\{W\in\R^{K\times d}\mid \|\Wv-\Wv_t\|_{H_t}^2\le \beta_t(\delta)\}$ is the confidence region, and $b_t(\x,\delta)$ is the optimistic bonus whose definition is deferred to \eqref{eq:oful-mlogb-optimistic-reward-difference} in Section~\ref{ssec:computation-obstacle}.
When $K=1$, we assume $(r_0,r_1)=(0,1)$, so the expected reward reduces to the scalar logistic reward $\sigma(\w^\top\x)$, and $\C_t$ reduces to the corresponding vector confidence set; see Appendix~\ref{ssec:apd-preliminary}.
These optimistic constructions ensure, with probability at least $1-\delta$, that $\rwdv^\top\sig(W_*\x)$ is upper-bounded by $R_t(\x,\delta)$ for all candidate actions $\x$.
We refer readers to Theorem~\ref{thm:confidence-region} and Lemma~\ref{lem:diff-rwd-upper-bound} for the corresponding guarantee.

Although OFUL-MLogB achieves a near-optimal regret bound $\Ot(Kd\sqrt{T})$ in the multinomial case \citep[Appendix~C.5]{zhang2023mlb}, its dependence on the full $Kd\times Kd$ Hessian leads to costly $\O(|\A_t|K^3d^2+K^3d^3)$ time and $\O(K^2d^2)$ space per round in high-dimensional settings.
Section~\ref{ssec:computation-obstacle} details these computational obstacles.


\subsection{Matrix Sketching and Frequent Directions} \label{ssec:fd}
Matrix sketching is a standard tool for reducing the cost of maintaining high-dimensional covariance matrices.
For a sketch size $m\ll d$ and a stream of vectors $\x_1,\ldots,\x_t\in\R^d$, a sketch maintains $S_t\in\R^{m\times d}$ such that $S_t^\top S_t\approx \sum_{s=1}^t \x_s\x_s^\top$.
Unlike randomized methods such as random projection \citep{baraniuk2010random-projection} or column selection \citep{drineas2006columnselection}, Frequent Directions (FD) \citep{ghashami2016fd} is deterministic and operates efficiently in a streaming fashion.

At step $t$, FD updates the sketch by appending $\x_t$ to $S_{t-1}$,
forming $\hat S_t=[S_{t-1}^\top,\x_t]^\top\in\R^{(m+1)\times d}$, and computes its SVD
\begin{equation}
	U\Sigma V^\top=\hat{S}_t
. \end{equation}
Since $\hat{S}_t^\top\hat{S}_t=V\Sigma^2V^\top=S_{t-1}^\top S_{t-1}+\x_t\x_t^\top$,
FD keeps the top $m$ right singular directions and shrinks the singular values:
\begin{equation}
	S_t = \sqrt{\Sigma_{1:m}^2-\sigma_{m+1}^2 I_m}\,V_{:,1:m}^\top
, \end{equation}
where $\sigma_{m+1}$ is the $(m+1)$-th singular value,
$\Sigma_{1:m}=\mathop{\mathrm{diag}}(\sigma_1,\dots,\sigma_m)$ contains the top $m$ singular values,
and $V_{:,1:m}$ contains the corresponding right singular vectors.

FD has been used in stochastic linear bandits \citep{kuzborskij2019efficient,chen2021efficient,wen2024currentpitfalls} and online convex
optimization \citep{luo2016ONSsketchig,luo2019rfd,wan2022adafd,yang2025dimensionfree}.
In stochastic linear bandits (SLB), sketching is effective because both the parameter update and the UCB bonus admit closed-form reductions via the Woodbury identity \citep{kuzborskij2019efficient}.
In contrast, the nonlinear multinomial logistic model in MLogB does not admit the same closed-form sketching reductions, so new spectral reductions are needed.
We detail the resulting computational obstacles in Section~\ref{ssec:computation-obstacle}.

\section{Efficient Algorithms for MLogB} \label{sec:fd-oful-mlogb}

We operate under the following standard assumptions \citep{amani2021ucb-based-MNL, zhang2023mlb}:
\begin{assumption} \label{asmp:boundedness}
The action set $\A_t$, the true parameter $\Ws$, and the reward vector $\rwdv \in \R^K_+$ are bounded:
$\sup_{\x \in \A_t} \|\x\|_2 \le D$, $\|\Ws\|_F \le S$, and $\|\rwdv\|_2 \le R$.
Without loss of generality, we set $r_0 = 0$ as a dummy outcome,
and ignore $0$-th dimension in the subsequent analysis since it does not contribute to the reward and the regret.
See Appendix~\ref{ssec:apd-preliminary} for more details.
\end{assumption}
\begin{assumption} \label{asmp:derivative-bound}
The multinomial logistic mapping $\sig(\cdot)$ satisfies strong convexity condition:
for any $W \in \W$ and $\x \in \A_t$, there exist constants $\kappa > 0$ and $L \in [0, 1]$ such that $\frac{1}{\kappa} I_K \preceq \nabla \sig(W \x) \preceq L I_K$.
\end{assumption}
The constant $\kappa$ typically depends exponentially on the domain diameter $SD$ \citep[Section 3]{amani2021ucb-based-MNL}.

\subsection{Computational Obstacles} \label{ssec:computation-obstacle}
We analyze the two Hessian-induced bottlenecks that hinder the scalability of OFUL-MLogB.

\paragraph{Constrained optimization for parameter estimation.}
The proximal ONS update \eqref{eq:oful-mlogb-W-update} can be decomposed into an unconstrained gradient step followed by an ellipsoidal projection:
\begin{align}
	\Wv' &= \Wv_t - \eta \Ht_t^{-1} \nabla \ell_t(W_t),
	\label{eq:oful-mlogb-W-gradient} \\
	W_{t+1} &= \argmin_{\|W\|_F \le S} \left\|\Wv-\Wv'\right\|_{\Ht_t}^2 ,
	\label{eq:oful-mlogb-W-projection}
\end{align}
where $\Ht_t = H_t+\eta\nabla^2\ell_t(W_t)$.
The gradient step requires solving a linear system involving the $Kd\times Kd$ matrix $\Ht_t$,
and the projection is a QCQP without a closed-form solution for a general Hessian.
Standard exact methods require a full eigen-decomposition \citep{Hazan2007ons}, leading to $\O(K^3d^3)$ time per round.
Moreover, simply replacing $H_t$ by a sketched matrix $S_t^\top S_t+\lambda I$ does not by itself yield an efficient projection,
which motivates our explicit use of the sketch's spectral form.

\paragraph{High-dimensional spectral norms in optimistic rewards.}
For $K\ge 2$, evaluating the optimistic reward $R_t(\x,\delta)$ in \eqref{eq:oful-mlogb-optimistic-reward} requires the bonus
\begin{equation} \label{eq:oful-mlogb-optimistic-reward-difference}
	b_t(\x, \delta)
	=
	\sqrt{\beta_t(\delta)}
	\left\|(\nabla \sig(W_t \x)\otimes \x)\rwdv\right\|_{H_t^{-1}}
	+
	\frac{1}{2}\|\rwdv\|_1\beta_t(\delta)
	\underbrace{
	\left\|H_t^{-\frac12}(I_K\otimes \x)\right\|_2^2
	}_{\text{spectral norm term}} .
\end{equation}
The dominant cost lies in the spectral norm term.
Specifically, computing the inverse square root $H_t^{-\frac{1}{2}}$ incurs $\Omega(K^2d^2)$ time ($\O(K^3d^3)$ via eigen-decomposition) and $\O(K^2d^2)$ space overhead.
Even after $H_t^{-\frac12}$ is formed, computing $\|H_t^{-\frac12}(I_K\otimes\x)\|_2$ still demands $\O(K^3d^2)$ time.
The eigenvalue-based implementation of OFUL-MLogB costs $\O(K^3d^3)$ time for each action \citep[C.3.4]{zhang2023mlb}.
Thus, action selection becomes impractical in high-dimensional settings.
The same type of bottleneck appears in other MLogB algorithms using second-order optimistic bonuses \citep{lee2025ImprovedConfidBound4MLogB,boudart2025enjoyNonLinear}.

\paragraph{Contrast with SLB.}
In stochastic linear bandits,
sketching is straightforward because the least-squares parameter update $\hat{\theta}=H_t^{-1}\sum_{s=1}^ty_t\x_t$ and the UCB bonus $\|\x\|_{H_t^{-1}}$ both admit closed-form Woodbury reductions to matrix-vector products \citep{kuzborskij2019efficient}.
MLogB lacks this structure:
the nonlinear model leads to a constrained second-order projection \eqref{eq:oful-mlogb-W-projection} for parameter estimation and high-dimensional spectral norms \eqref{eq:oful-mlogb-optimistic-reward-difference} in the optimistic bonus.
These two obstacles preclude a direct application of standard sketching techniques and motivate the spectral reductions developed below.

\subsection{EOFD-MLogB} \label{ssec:fd-oful-mlogb}
We propose EOFD-MLogB (\textbf{E}fficient \textbf{O}ptimism via \textbf{F}requent \textbf{D}irections),
whose time and space complexities scale linearly with the dimension $d$.

\subsubsection{Spectral Maintenance via Adapted FD}
Unlike prior works that maintain the sketch matrix directly \citep{kuzborskij2019efficient,luo2016ONSsketchig},
EOFD-MLogB maintains a low-rank spectral decomposition of the sketched Hessian.
We approximate $H_t=\lambda I_{Kd}+\sum_{s=1}^{t-1}\nabla^2\ell_s(W_{s+1})$ by
\begin{equation}
	Z_t=\lambda I_{Kd}+V_t\Lambda_tV_t^\top
, \end{equation}
where $\Lambda_t\in\R^{m\times m}$ is diagonal and positive semidefinite,
and $V_t\in\R^{Kd\times m}$ has orthonormal columns.
This spectral form is crucial for the efficient parameter update and optimistic reward computation below.

Standard FD is designed for vector streams, while each Hessian update in MLogB is a rank-$K$ matrix: $\nabla^2\ell_t(W_{t+1}) = \nabla\sig(W_{t+1}\x_t)\otimes \x_t\x_t^\top$ (cf.~Appendix~\ref{ssec:apd-preliminary}).
To apply FD, we perform eigen-decomposition
\begin{equation}
	PDP^\top = \mathop{\mathrm{EIG}}(\nabla\sig(W_{t+1}\x_t))
. \end{equation}
Then the Hessian update admits the block factorization
\begin{equation}
	\nabla^2\ell_t(W_{t+1}) = B_tB_t^\top
	, \qquad
	B_t = P D^{1/2}\otimes \x_t \in \R^{Kd\times K}
. \end{equation}
EOFD-MLogB performs a single block FD update by applying FD to $[V_t\Lambda_t^{1/2}, B_t]\in\R^{Kd\times(m+K)}$.
Specifically, compute
\begin{equation}
	\bar U\bar\Sigma\bar V^\top = \mathop{\mathrm{SVD}}\!\left([V_t\Lambda_t^{1/2},B_t]^\top\right)
, \end{equation}
and update
\begin{equation}
	\Lambda_{t+1} = \bar\Sigma_{1:m}^2-\bar\sigma_{m+1}^2I_m
	, \qquad
	V_{t+1} = \bar V_{:,1:m}
. \end{equation}
where $\bar\sigma_{m+1}$ is the $(m+1)$-th singular value,
$\bar\Sigma_{1:m}=\mathop{\mathrm{diag}}(\sigma_1,\dots,\sigma_m)$ contains the top $m$ singular values,
and $\bar V_{:,1:m}$ contains the corresponding right singular vectors.
The discarded mass $\bar\sigma_{m+1}^2$ is accumulated into the sketching error $\rho_{1:t}=\rho_{1:t-1}+\bar\sigma_{m+1}^2$.
The procedure is summarized in Algorithm~\ref{algo:adpfd}.

\begin{algorithm}[tb]
\caption{Adapted FD for Matrix-Valued Updates} \label{algo:adpfd}
\KwIn{Sketch size $m$, Hessian update $\nabla^2 \ell_t(W_{t+1})=\nabla\sig(W_{t+1}\x_t)\otimes \x_t\x_t^\top$, current sketch $(V_t,\Lambda_t,\rho_{1:t-1})$}
$PDP^\top = \mathop{\mathrm{EIG}}(\nabla\sig(W_{t+1}\x_t))$ \\
$B_t = PD^{\frac{1}{2}}\otimes \x_t$ \\
$\bar U\bar\Sigma\bar V^\top = \mathop{\mathrm{SVD}}([V_t\Lambda_t^{\frac12},B_t]^\top)$ \\
$\bar\sigma_{m+1} = \bar\Sigma_{m+1,m+1}$ \\
$\Lambda_{t+1} = \bar\Sigma_{1:m}^2-\bar\sigma_{m+1}^2I_m$ \\
$V_{t+1} = \bar V_{:,1:m}$ \\
$\rho_{1:t} = \rho_{1:t-1}+\bar\sigma_{m+1}^2$ \\
\Return{$(V_{t+1},\Lambda_{t+1},\rho_{1:t})$}
\end{algorithm}

\subsubsection{Optimistic Reward Construction and Parameter Update}
EOFD-MLogB replaces the full Hessian $H_t=\lambda I_{Kd}+\sum_{s=1}^{t-1}\nabla^2\ell_s(W_{s+1})$ by the sketched Hessian $Z_t$ in both parameter estimation \eqref{eq:oful-mlogb-W-update} and optimistic reward construction \eqref{eq:oful-mlogb-optimistic-reward}.

In each iteration $t$, given the current estimator $W_t$, the algorithm selects an action $\x_t$ according to the optimistic principle.
In the binary case ($K=1$), the decision rule simplifies to a linear estimation term plus a bonus vector norm term \citep{zhang2016logistic} due to the monotonicity of $\sigma(\cdot) \in (0, 1)$,
\begin{equation}
	\x_t = \argmax_{\x\in\A_t} \max_{\w\in\C_t} \sigma(\w^\top\x)
	\iff
	\x_t = \argmax_{\x\in\A_t} \w_t^\top \x + \sqrt{\beta_t(\delta)} \|\x\|_{Z_t^{-1}}
. \end{equation}
The exploration bonus $\left\| \x \right\|_{Z_t^{-1}}$ can be computed efficiently using the Woodbury identity:
$Z_t^{-1} = \frac{1}{\lambda} (I_d - V_t \hat{\Lambda}_t V_t^\top)$, where $\hat{\Lambda}_t = \Lambda_t(\lambda I_m + \Lambda_t)^{-1}$.
In the multinomial case ($K \ge 2$),
the first norm term in \eqref{eq:oful-mlogb-optimistic-reward-difference} is computed using the same Woodbury identity for $Z_t^{-1}$,
and the second term requires evaluating the spectral norm $\|Z_t^{-\frac{1}{2}}(I_K\otimes\x)\|_2^2$.
The following lemma exploits the spectral form $Z_t=\lambda I+V_t\Lambda_tV_t^\top$ to reduce this term to the spectral norm of a $K\times K$ matrix.
\begin{lemma} \label{lem:fd-oful-mlogb-efficient-reward}
For any $\x \in \A_t$, let $Z_t = \lambda I_{Kd} + V_t \Lambda_t V_t^\top$.
The spectral norm satisfies
\begin{equation}
	\left\| Z_t^{-\frac{1}{2}} (I_K \otimes \x) \right\|_2^2 = \|A(\x)\|_2
, \end{equation}
where $A(\x) \in \R^{K \times K}$ is a symmetric matrix with entries given by
\begin{equation} \label{eq:fd-oful-mlogb-efficient-reward-A}
	A(\x)_{ij}
	=
	\sum_{k=1}^{m} \left(\frac{1}{\mu_k + \lambda} - \frac{1}{\lambda}\right) \left< \x, \bu_{ki} \right> \left< \x, \bu_{kj} \right>
	+
	\frac{\|\x\|_2^2}{\lambda} \delta_{ij}
. \end{equation}
Here, $\mu_k$ is the $k$-th diagonal element of $\Lambda_t$.
$\bu_{ki} \in \R^d$ denotes the sub-vector extracted from indices $(i-1)d+1$ to $id$ of $V_t$'s $k$-th column,
and $\delta_{ij}$ denotes the Kronecker delta.
\end{lemma}
Based on Lemma~\ref{lem:fd-oful-mlogb-efficient-reward}, the optimistic reward becomes
\begin{equation} \label{eq:fd-oful-mlogb-efficient-reward}
	R_t(\x, \delta)
	= \rwdv^\top \sig(W_t \x) + \sqrt{\beta_t(\delta)} \left\| (\nabla \sig(W_t \x) \otimes \x) \rwdv \right\|_{Z_t^{-1}}
	+ \frac{1}{2} \|\rwdv\|_1 \beta_t(\delta) \|A(\x)\|_2
. \end{equation}

After submitting $\x_t$ and receiving feedback $y_t$, the algorithm updates the parameter estimate.
Since the parameter update in \eqref{eq:oful-mlogb-W-gradient} and \eqref{eq:oful-mlogb-W-projection} involves $\Ht_t = H_t + \eta \nabla^2 \ell_t(W_t)$,
the algorithm maintains a surrogate Hessian $\Zt_t = Z_t + \eta \nabla^2 \ell_t(W_t)$ and decomposes it into spectral form.
Similar to Algorithm~\ref{algo:adpfd}, we compute the eigen-decomposition of $\nabla \sig(W_t \x_t)$ and then perform an SVD on the concatenated matrix:
\begin{align}
	\nabla \sig(W_t \x_t) &= P D P^\top
	\label{eq:fd-oful-mlogb-zt-decomposition}
	\\
	\Ut_t \Sigt_t \Vt_t^\top
	&= \mathop{\mathrm{SVD}}\left([V_t \Lambda_t^{\frac{1}{2}}, P (\eta D)^{\frac{1}{2}} \otimes \x_t]^\top\right)
	\label{eq:fd-oful-mlogb-zt-svd}
	\\
	\Lamt_t &= \Sigt_t^2
	\label{eq:fd-oful-mlogb-zt-lambda}
\end{align}
This yields the decomposition $\Zt_t = \lambda I_{Kd} + \Vt_t \Lamt_t \Vt_t^\top$.
Here, $\Vt_t\in\R^{Kd\times(m+K)}$ has orthonormal columns and $\Lamt_t\in\R^{(m+K)\times(m+K)}$ is diagonal.

The gradient step \eqref{eq:oful-mlogb-W-gradient} is computed as $\Wv' = \Wv_t - \eta \Zt_t^{-1} \nabla \ell_t(W_t)$,
which can be efficiently calculated using the Woodbury identity $\Zt_t^{-1} = \frac{1}{\lambda} \left( I_{Kd} - \Vt_t \Lamt_t (\lambda I_{m+K} + \Lamt_t)^{-1} \Vt_t^\top \right)$.
To avoid the high-dimensional projection in \eqref{eq:oful-mlogb-W-projection},
the algorithm leverages the spectral structure of $\Zt_t = \lambda I + \Vt_t \Lamt_t \Vt_t^\top$ to reduce it to a one-dimensional root-finding problem,
as formalized in the following lemma:
\begin{lemma} \label{lem:fd-oful-mlogb-efficient-W-update}
Let $\Zt_t = \lambda I_{Kd} + \Vt_t \Lamt_t \Vt_t^\top$, where $\bvt_i$ is the $i$-th column of $\Vt_t$ and $\mu_i$ is the $i$-th diagonal
element of $\Lamt_t$.
If $\|\Wv'\|_2 \le S$, then $\Wv_{t+1} = \Wv'$.
Otherwise, the projection is given by:
\begin{equation} \label{eq:fd-oful-mlogb-efficient-W-update-W}
	\Wv_{t+1} = \left(\frac{\lambda}{\lambda + \nu} I + \frac{\nu}{\lambda + \nu} \Vt_t \Lamt_t M \Vt_t^\top\right) \Wv'
, \end{equation}
where $M = ((\lambda + \nu)I_{m+K} + \Lamt_t)^{-1}$ is diagonal.
The dual variable $\nu > 0$ is the unique root of
\begin{equation} \label{eq:fd-oful-mlogb-efficient-W-update-nu}
	f(\nu)= S^2
, \end{equation}
where $f(\nu)$ is a monotonically decreasing function on the interval $(0, +\infty)$
\begin{equation}
	f(\nu) = \frac{(\|\Wv'\|_2^2-\|\Vt_t^\top\Wv'\|_2^2) \lambda^2}{(\lambda + \nu)^2} + \sum_{i=1}^{m+K} \frac{(\bvt_i^\top \Wv')^2 (\mu_i + \lambda)^2}{(\mu_i + \lambda + \nu)^2}
. \end{equation}
\end{lemma}
The coefficients in \eqref{eq:fd-oful-mlogb-efficient-W-update-nu} can be precomputed before root-finding.
Since $f(\nu)$ is a scalar function, $\nu$ can be found efficiently via Newton's method or bisection, and errors can be controlled within machine precision.
Algorithm~\ref{algo:fd-oful-mlogb} summarizes the complete EOFD-MLogB procedure.

\begin{algorithm}[tb]
\caption{EOFD-MLogB} \label{algo:fd-oful-mlogb}
\KwIn{regularization parameter $\lambda$, confidence level $\delta$, step size $\eta$, parameter domain diameter $S$, sketch size $m$}
Initialize $W_1 = 0$, $\Lambda_1 = 0_{m\times m}$, arbitrary orthonormal $V_1\in\R^{Kd\times m}$, $\rho_{1:0} = 0$; define $Z_t=\lambda I_{Kd}+V_t\Lambda_tV_t^\top$. \\
\For{$t = 1, 2, \dots, T$}{
	\eIf{$K = 1$}{
		$\x_t = \argmax_{\x \in \A_t} \w_t^\top \x + \sqrt{\beta_t(\delta)} \|\x\|_{Z_t^{-1}}$
	}{
		Construct $A(\x)$ by \eqref{eq:fd-oful-mlogb-efficient-reward-A} for all $\x \in \A_t$, and choose $\x_t = \argmax_{\x \in \A_t} R_t(\x, \delta)$ by \eqref{eq:fd-oful-mlogb-efficient-reward}
	}
	Submit $\x_t$ and receive feedback $y_t$ \\
	Form $\Zt_t = \lambda I_{Kd} + \Vt_t \Lamt_t \Vt_t^\top$ by \eqref{eq:fd-oful-mlogb-zt-decomposition}--\eqref{eq:fd-oful-mlogb-zt-lambda} \\
	$\Wv' = \Wv_t - \eta \Zt_t^{-1} \nabla \ell_t(W_t)$ \\
	\eIf{$\|\Wv'\|_2 \le S$}{
		$\Wv_{t+1} = \Wv'$
	}{
		Find root $\nu$ of \eqref{eq:fd-oful-mlogb-efficient-W-update-nu}, and
		update $\Wv_{t+1}$ by \eqref{eq:fd-oful-mlogb-efficient-W-update-W}
	}
	Pass $\nabla^2 \ell_t(W_{t+1})$ to Algorithm~\ref{algo:adpfd} to update the sketch $(\Lambda_{t+1}, V_{t+1}, \rho_{1:t})$
}
\end{algorithm}

\paragraph{Computational Complexity.}
For space complexity, OFUL-MLogB stores the full $Kd\times Kd$ Hessian and hence requires $\O(K^2d^2)$ space.
In contrast, EOFD-MLogB maintains the sketched decompositions $(V_t,\Lambda_t)$ and $(\Vt_t,\Lamt_t)$, using $\O(Kd(m+K))$ space.
For time complexity, OFUL-MLogB costs $\O(|\A_t|K^3d^2+K^3d^3)$ per round, as discussed in Section~\ref{ssec:computation-obstacle}.
EOFD-MLogB evaluates each candidate action in $\O(Kdm+K^2m+K^3)$ time, while both the parameter update and the block FD update cost $\O(Kd(m+K)^2)$.
Thus, the total per-round time complexity is $\O(|\A_t|(Kdm+K^2m+K^3)+Kd(m+K)^2))$.
If we further emphasize the dependence on $d$, $m$ and $K$, the time complexity can be expressed as $\O(Kd(m+K)^2)$, which is linear in $d$, compared with the cubic dependence of OFUL-MLogB.

\subsection{Regret Analysis} \label{ssec:regret-analysis}
In this section, we establish theoretical guarantees for EOFD-MLogB.
Let $\tail{t}=1+\rho_{1:t}/\lambda$, where $\rho_{1:t}$ is the cumulative FD sketching error associated with $Z_t$, defined in line 7 of Algorithm~\ref{algo:adpfd}.
We first establish a confidence set for the unknown parameter $W_*$.

\begin{theorem}[Confidence Region] \label{thm:confidence-region}
Let $\alpha = 2 (1+SD)+\ln(K+1)$, and $\tail{t} = 1 + \rho_{1:t} / \lambda$.
Set $\lambda \ge \max\{2 L D^2, \alpha \max\{2 \sqrt6 S L D^2, 24 Kd\}\}$ and $\eta = \alpha/2$.
For any $\delta \in (0, 1)$, with probability at least $1 - \delta$:
\begin{equation}
	\Ws \in \C_t = \left\{\, W \in \W \,\middle|\, \left\| \Wv - \Wv_t \right\|_{Z_t}^2 \le \beta_t(\delta) \,\right\}, \quad \forall t \ge 1
, \end{equation}
where the radius
\begin{equation}
	\beta_t(\delta) = \O\left(
		\tail{t} \ln^2 K \ln t
		( m \ln t + Kd \ln \tail{t} + \ln \frac{1}{\delta} )
	\right)
. \end{equation}
\end{theorem}
The precise definition of $\beta_t(\delta)$ is given in \eqref{eq:cr-beta} in appendix,
and the detailed proof is provided in Appendix~\ref{ssec:proofthm}.
The confidence region is defined with respect to the sketched matrix $Z_t$ instead of the full Hessian $H_t$.
Compared with the OFUL-MLogB confidence radius $\tilde{\beta}_t(\delta)=\O(\ln^2K\ln t(Kd\ln t+\ln\frac1\delta))$ \citep[Theorem 3]{zhang2023mlb},
the sketched radius replaces $Kd\ln t$ by $Kd\ln\tail{t}+m\ln t$ and introduces the multiplicative factor $\tail{t}$.
These terms quantify the cumulative approximation error introduced by FD.

\begin{remark}[Independence of $\kappa$]
Theorem~\ref{thm:confidence-region} preserves the $\kappa$-independence of OFUL-MLogB.
This follows because EOFD-MLogB keeps the same second-order surrogate in \eqref{eq:oful-mlogb-W-update} and index-shifted Hessian $H_t=\lambda I+\sum_{s=1}^{t-1}\nabla^2 \ell_s(W_{s+1})$ as OFUL-MLogB,
while the singular values discarded by FD are controlled by the upper bound $\O(LD^2)$ on Hessian updates rather than the strong-convexity lower bound $1/\kappa$.
\end{remark}

We next derive the regret bound for the multinomial case based on Theorem~\ref{thm:confidence-region}.
\begin{theorem}[Multinomial Regret] \label{thm:regret-K-ge-2}
Under the conditions of Theorem~\ref{thm:confidence-region}, with probability at least $1-\delta$,
\begin{equation}
	\reg(T)
	=
	\Ot\left(
		\tail{T} (K d \ln \tail{T} + m) \sqrt{T}
	\right)
. \end{equation}
\end{theorem}

Theorem~\ref{thm:regret-K-ge-2} achieves similar regret bound of OFUL-MLogB $\Ot(Kd\sqrt{T})$ \citep[Theorem 4]{zhang2023mlb} except the sketching factor $\tail{T}$ and the replacement of $Kd$ by $Kd\ln\tail{T}+m$.
When $m\ge Kd$, no truncation occurs, so $\rho_{1:T}=0$ and $\tail{T}=1$, recovering the OFUL-MLogB bound.

The sketching error is controlled by the spectral tail of the cumulative Hessian.
By the FD guarantee \citep[Theorem 1.1]{ghashami2016fd},
\begin{equation}
	\rho_{1:T} \le \min_{0\le k<m} \frac{1}{m-k} \left\|M_{T+1}^{1/2}-(M_{T+1}^{1/2})_{(k)}\right\|_F^2
, \end{equation}
where $M_T=\sum_{t=1}^T\nabla^2\ell_t(W_{t+1})$ denotes the cumulative Hessian and $(M_T^{1/2})_{(k)}$ is the best rank-$k$ approximation of $M_T^{1/2}$.
Thus, if the accumulated Hessian $M_T$ is approximately low-rank, EOFD-MLogB remains close to OFUL-MLogB in regret while substantially reducing computation.
If the Hessian has slow spectral decay, the sketching error may grow and the regret bound correspondingly deteriorates.

\begin{remark}[Extension to Related MLogB Frameworks] \label{rmk:extension}
The proposed spectral reductions are not specific to OFUL-MLogB.
For example, \citet{boudart2025enjoyNonLinear} use a similar second-order optimistic bonus and ONS-type update, differing mainly in the initialization phase.
Replacing the full Hessian by $Z_t$ and applying Lemma~\ref{lem:fd-oful-mlogb-efficient-reward} and Lemma~\ref{lem:fd-oful-mlogb-efficient-W-update} yields analogous efficiency gains and regret guarantees.
We focus on OFUL-MLogB because it naturally handles time-varying action sets $\A_t$.
\end{remark}

For completeness, we also state the binary-case guarantee, which admits the simplified optimistic rule and tighter regret bound.
\begin{proposition}[Binary Regret] \label{prop:regret-k-eq-1}
Under the conditions of Theorem~\ref{thm:confidence-region},
if $K = 1$, with probability at least $1-\delta$,
\begin{equation}
	\reg(T)
	=
	\Ot\left(
		\tail{T} (d \ln \tail{T} + m) \sqrt{T/\kappa_*}
	\right)
, \end{equation}
where $\kappa_* = 1 / \sigma'(\ws^\top\x_*)$, and $\x_*$ is the optimal action.
\end{proposition}

The binary case follows \citet[Corollary 1]{zhang2023mlb}:
the monotonicity of the scalar logistic reward reduces the optimistic decision rule to a simple vector norm, avoiding the multinomial spectral-norm bonus for the second-order reward approximation.
Theoretically, it yields the refined $\sqrt{1/\kappa_*}$ dependence, while retaining the sketching factor $\tail{T}$.

\section{Experiments} \label{sec:experiment}

\begin{figure*}[tb]
	\centering
	\begin{subfigure}{0.32\textwidth}
		\includegraphics[width=\linewidth]{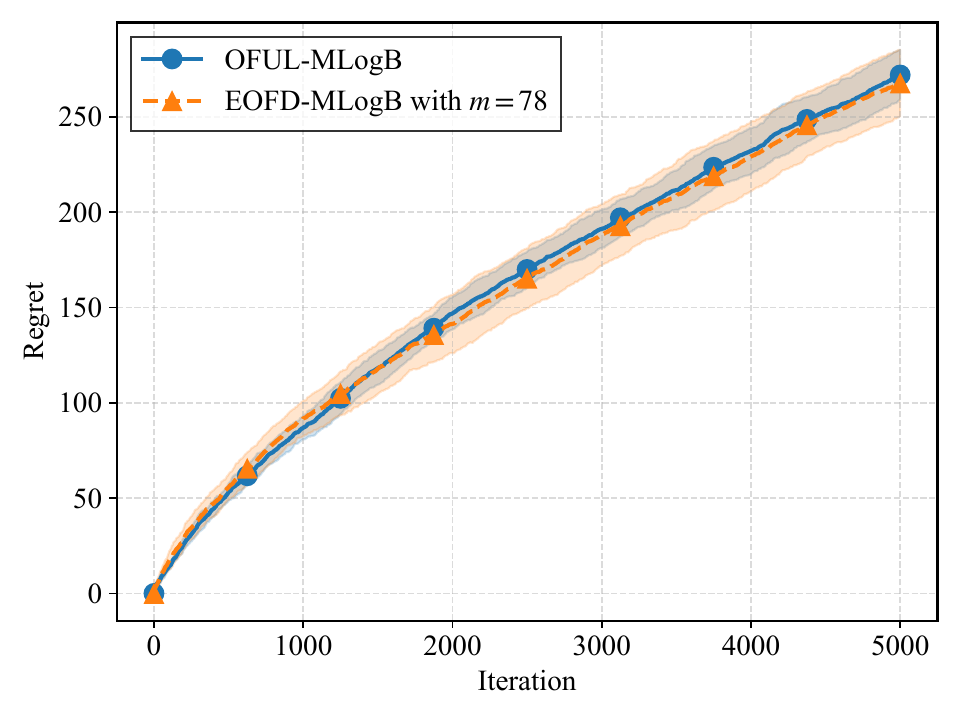}
		\caption{\texttt{MNIST}}
		\label{fig:exp-mnist-regret}
	\end{subfigure}
	\hfill
	\begin{subfigure}{0.32\textwidth}
		\includegraphics[width=\linewidth]{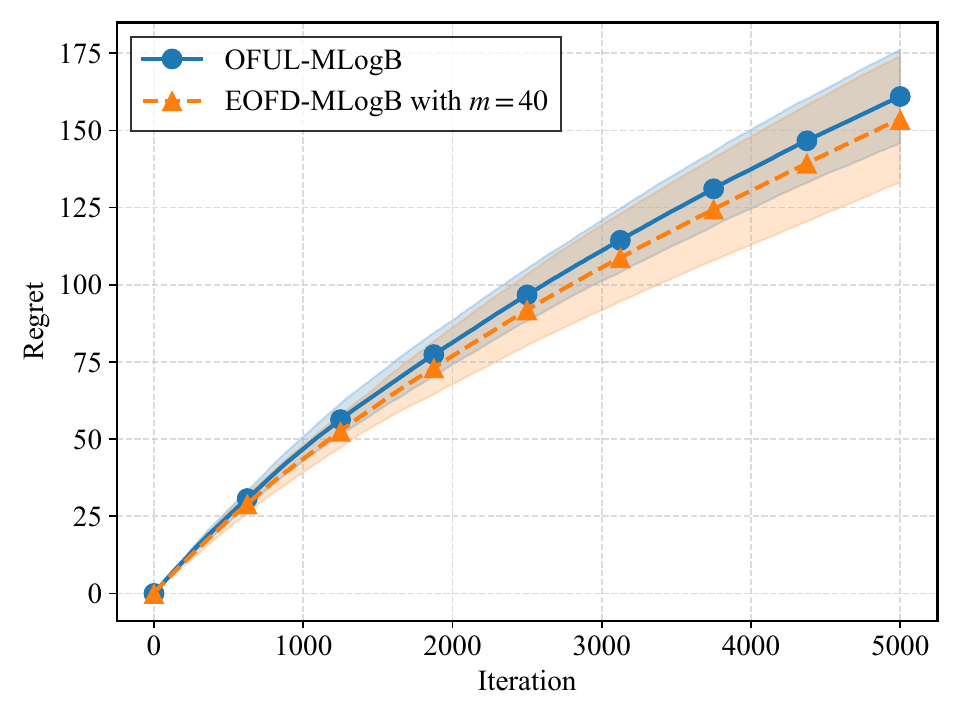}
		\caption{Multinomial case}
		\label{fig:exp-mul-regret}
	\end{subfigure}
	\hfill
	\begin{subfigure}{0.32\textwidth}
		\includegraphics[width=\linewidth]{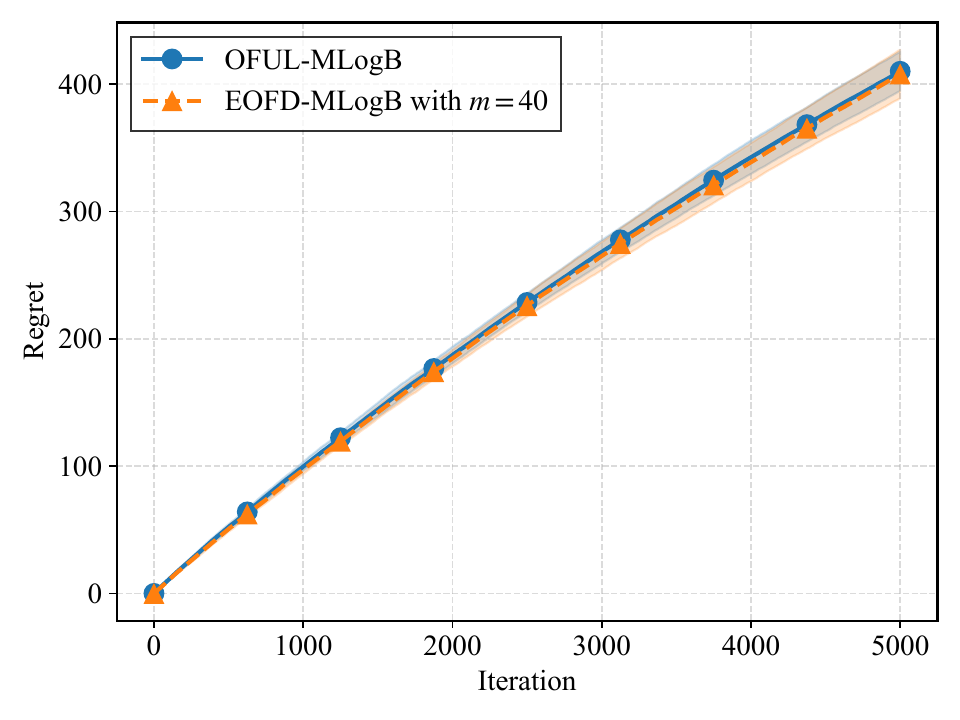}
		\caption{Binary case}
		\label{fig:exp-bin-regret}
	\end{subfigure}
	\caption{Cumulative regret of EOFD-MLogB and OFUL-MLogB over 10 trials.}
	\label{fig:exp-reg}
\end{figure*}

\begin{figure*}[tb]
	\centering
	\begin{subfigure}{0.32\textwidth}
		\includegraphics[width=\linewidth]{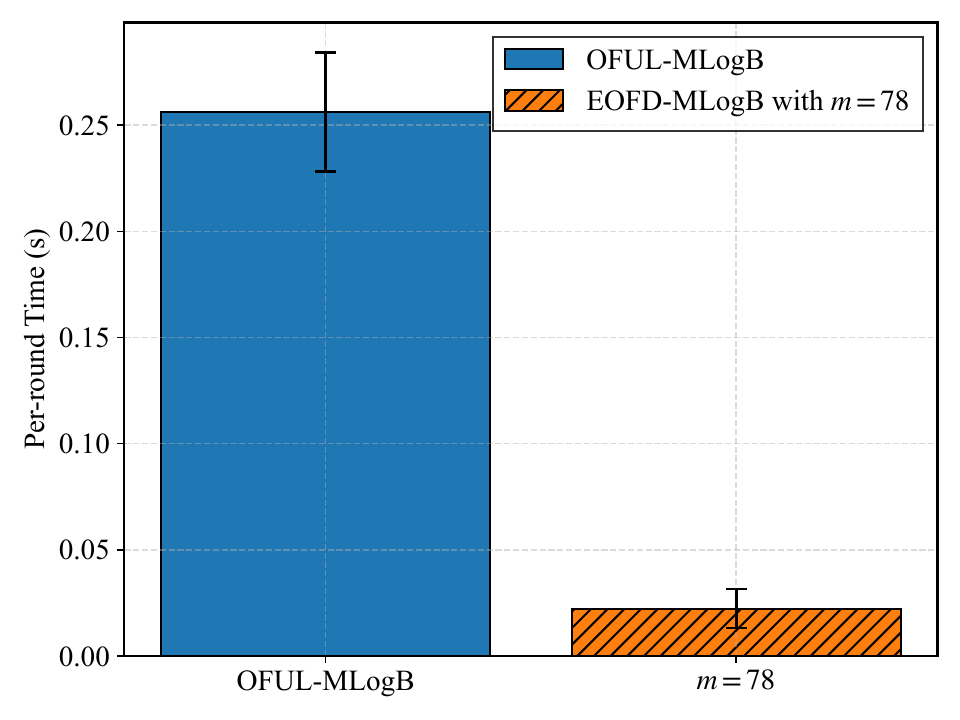}
		\caption{\texttt{MNIST}}
		\label{fig:exp-mnist-time}
	\end{subfigure}
	\hfill
	\begin{subfigure}{0.32\textwidth}
		\includegraphics[width=\linewidth]{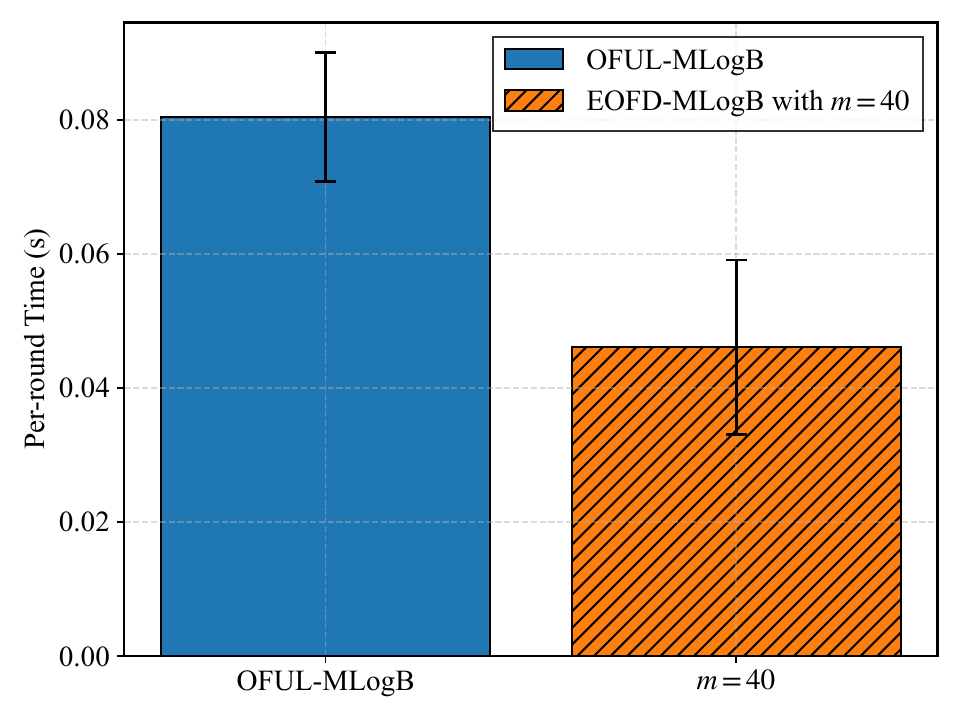}
		\caption{Multinomial case}
		\label{fig:exp-mul-time}
	\end{subfigure}
	\hfill
	\begin{subfigure}{0.32\textwidth}
		\includegraphics[width=\linewidth]{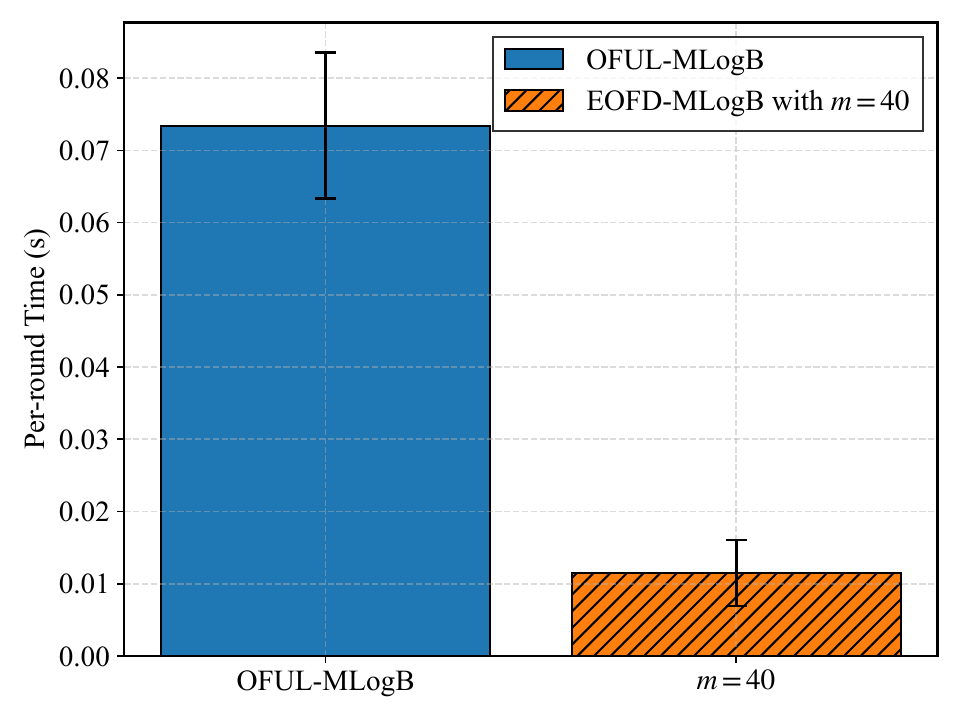}
		\caption{Binary case}
		\label{fig:exp-bin-time}
	\end{subfigure}
	\caption{Average per-round computation time over 10 trials.}
	\label{fig:exp-time}
\end{figure*}

This section evaluates the computational efficiency and regret performance of EOFD-MLogB on \texttt{MNIST} \citep{lecun1998gradient} and synthetic data.

\paragraph{Setup.}
We follow the experimental configuration of OFUL-MLogB \citep{zhang2023mlb}.
The baseline hyperparameters are $\lambda_0=\sqrt{K}\alpha S$ and $\eta_0=S/2+\ln(K+1)/4$.
We set $(\lambda,\eta)=((\sqrt{K}d)^l\lambda_0,(\sqrt{K}d)^e\eta_0)$ and search over $(l,e)\in\{-0.5,-0.25,0,0.25,0.5,0.75,1\}^2$ on synthetic validation runs.
For the reported runs, we use $(l,e)=(0.75,0.75)$ based on the convergence of $W_t$ for both EOFD-MLogB and OFUL-MLogB.
Unless otherwise stated, we set $m=\lfloor Kd/10\rfloor$, and $\delta=1/T$.
All results are averaged over 10 independent trials.

For \texttt{MNIST}, we adopt the protocol of \citet{cesa2013gang} to convert the multi-class problem into a binary contextual bandit problem.
At each round, the learner receives 10 candidate with distinct labels,
and obtains reward 1 if the selected label is \texttt{1} and 0 otherwise.
We append a bias term to the normalized features, yielding $d=785$.
For synthetic experiments, we simulate high-dimensional environments with $(d,K)=(100,4)$ and $(400,1)$.
At each round, action sets of size $|\A_t|\in[10,30]$ are sampled from the unit sphere.
For the multinomial synthetic setting, the reward vector is $\rwdv=[0,0.25,0.5,0.75,1]$.

The experiments are run on Intel(R) Core(TM) i5-14500 CPU (20 cores, 0.8 GHz min, 5.0 GHz max) with 16GB RAM, and the code is implemented in Python 3.12.13, with single thread.

\paragraph{Results.}
Figures~\ref{fig:exp-reg} and~\ref{fig:exp-time} report cumulative regret and average per-round computation time, respectively.
Solid lines and bars show means over 10 trials; shaded regions and error bars show one sample standard deviation.

On \texttt{MNIST}, EOFD-MLogB reduces the per-round computation time of OFUL-MLogB by approximately $80\%$ while maintaining a similar regret curve, with about $10\%$ average regret overhead
(Figures~\ref{fig:exp-mnist-time} and~\ref{fig:exp-mnist-regret}).
On synthetic data, EOFD-MLogB achieves regret comparable to OFUL-MLogB in both multinomial and binary settings
(Figures~\ref{fig:exp-mul-regret} and~\ref{fig:exp-bin-regret}).
Meanwhile, it reduces per-round computation time by over $35\%$ and $80\%$ in the multinomial and binary settings, respectively
(Figures~\ref{fig:exp-mul-time} and~\ref{fig:exp-bin-time}).
These results demonstrate that EOFD-MLogB substantially improves computational efficiency while preserving competitive regret.

\section{Conclusion} \label{sec:conclusion}

In this paper, we proposed EOFD-MLogB to address the prohibitive full-Hessian time and space costs of OFUL-MLogB in high-dimensional settings.
By maintaining a low-rank FD sketch of the accumulated Hessian, EOFD-MLogB reduces parameter estimation and optimistic reward construction to efficient spectral subroutines.
It achieves dominant per-round time complexity $\O(Kd(m+K)^2)$ and space complexity $\O(Kd(m+K))$, where $m\ll d$ is the sketch size.
Theoretically, we established a regret bound $\Ot(\tail{T}(Kd\ln\tail{T}+m)\sqrt{T})$, where $\tail{T}=1+\rho_{1:T}/\lambda$ and $\rho_{1:T}$ is controlled by the spectral tail of the cumulative Hessian.
This characterizes the FD approximation error:
when the Hessian is approximately low-rank, EOFD-MLogB maintains regret close to OFUL-MLogB while substantially reducing computation.
Experiments on real and synthetic data confirm that EOFD-MLogB delivers significant speedups with competitive regret.

The regret guarantee depends on the spectral decay of the accumulated Hessian;
for slower spectral decay, the sketching factor $\tail{T}$ may become large.
Future work includes adaptive selection of the sketch size $m$ and developing problem-dependent theoretical conditions or adaptive algorithms that directly control $\tail{T}$.

\bibliographystyle{plainnat}
\bibliography{ref}

@inproceedings{tran-dinh2015composite-self-concordant,
	title = {Composite Convex Minimization Involving Self-concordant-Like Cost
	         Functions},
	author = {Tran-Dinh, Quoc and Li, Yen-Huan and Cevher, Volkan},
	booktitle = {Proceedings of the 3rd International Conference on Modelling,
	             Computation and Optimization in Information Systems and
	             Management Sciences},
	pages = {155--168},
	year = {2015},
}

@inproceedings{zhang2016logistic,
	title = {Online Stochastic Linear Optimization under One-bit Feedback},
	author = {Zhang, Lijun and Yang, Tianbao and Jin, Rong and Xiao, Yichi and
	          Zhou, Zhi-hua},
	booktitle = {Proceedings of The 33rd International Conference on Machine
	             Learning},
	pages = {392--401},
	year = {2016},
}

@inproceedings{abeille2021minimax,
	title = {Instance-Wise Minimax-Optimal Algorithms for Logistic Bandits},
	author = {Abeille, Marc and Faury, Louis and Calauzenes, Clement},
	booktitle = {Proceedings of The 24th International Conference on Artificial
	             Intelligence and Statistics},
	pages = {3691--3699},
	year = {2021},
}

@inproceedings{amani2021ucb-based-MNL,
	title = {{UCB}-based Algorithms for Multinomial Logistic Regression Bandits},
	author = {Amani, Sanae and Thrampoulidis, Christos},
	booktitle = {Advances in Neural Information Processing Systems},
	volume = {34},
	pages = {2913--2924},
	year = {2021},
}

@inproceedings{zhang2023mlb,
	author = {Zhang, Yu-Jie and Sugiyama, Masashi},
	title = {Online (Multinomial) Logistic Bandit: Improved Regret and Constant
	         Computation Cost},
	booktitle = {Advances in Neural Information Processing Systems},
	volume = {36},
	pages = {29741--29782},
	year = {2023},
}

@inproceedings{zhang2026GLBonepass,
	title = {Generalized Linear Bandits: Almost Optimal Regret with One-Pass
	         Update},
	author = {Zhang, Yu-Jie and Xu, Sheng-An and Zhao, Peng and Sugiyama,
	          Masashi},
	booktitle = {Advances in Neural Information Processing Systems},
	volume = {38},
	year = {2025},
}

@inproceedings{jezequel2021mixiability,
	title = {Mixability made efficient: Fast online multiclass logistic
	         regression},
	author = {Jézéquel, Rémi and Gaillard, Pierre and Rudi, Alessandro},
	booktitle = {Advances in Neural Information Processing Systems},
	volume = {34},
	pages = {23692--23702},
	year = {2021},
}

@article{drineas2006columnselection,
	title = {Fast Monte Carlo Algorithms for Matrices I: Approximating Matrix
	         Multiplication},
	author = {Drineas, Petros and Kannan, Ravi and Mahoney, Michael W.},
	journal = {SIAM Journal on Computing},
	volume = {36},
	number = {1},
	pages = {132--157},
	year = {2006},
}

@inproceedings{kuzborskij2019efficient,
	title = {Efficient Linear Bandits through Matrix Sketching},
	author = {Kuzborskij, Ilja and Cella, Leonardo and Cesa-Bianchi, Nicol\`{o}},
	booktitle = {Proceedings of the 22nd International Conference on Artificial
	             Intelligence and Statistics},
	pages = {177--185},
	year = {2019},
}

@inproceedings{cesa2013gang,
	title = {A Gang of Bandits},
	author = {Cesa-Bianchi, Nicolò and Gentile, Claudio and Zappella, Giovanni},
	booktitle = {Advances in Neural Information Processing Systems},
	volume = {26},
	year = {2013},
	pages = {737--745},
}

@inproceedings{chen2021efficient,
	title = {Efficient and Robust High-Dimensional Linear Contextual Bandits},
	author = {Chen, Cheng and Luo, Luo and Zhang, Weinan and Yu, Yong and Lian,
	          Yijiang},
	booktitle = {Proceedings of the 29th International Joint Conference on
	             Artificial Intelligence},
	pages = {4259--4265},
	year = {2020},
}

@inproceedings{filippi2010glb,
	title = {Parametric Bandits: The Generalized Linear Case},
	author = {Filippi, Sarah and Cappe, Olivier and Garivier, Aur\'{e}lien and
	          Szepesv\'{a}ri, Csaba},
	booktitle = {Advances in Neural Information Processing Systems},
	volume = {23},
	pages = {586--594},
	year = {2010},
}

@inproceedings{faury2022jointly-logistic-bandit,
	title = { Jointly Efficient and Optimal Algorithms for Logistic Bandits },
	author = {Faury, Louis and Abeille, Marc and Jun, Kwang-Sung and Calauzenes,
	          Clement},
	booktitle = {Proceedings of The 25th International Conference on Artificial
	             Intelligence and Statistics},
	volume = {151},
	pages = {546--580},
	year = {2022},
}

@misc{wen2024currentpitfalls,
	title = {Matrix Sketching in Bandits: Current Pitfalls and New Framework},
	author = {Dongxie Wen and Hanyan Yin and Xiao Zhang and Zhewei Wei},
	year = {2024},
	eprint = {2410.10258},
	archivePrefix = {arXiv},
}

@article{ghashami2016fd,
	title = {Frequent Directions: Simple and Deterministic Matrix Sketching},
	author = {Ghashami, Mina and Liberty, Edo and Phillips, Jeff M. and Woodruff
	          , David P.},
	journal = {SIAM Journal on Computing},
	volume = {45},
	number = {5},
	pages = {1762--1792},
	year = {2016},
}

@article{baraniuk2010random-projection,
	title = {Low-Dimensional Models for Dimensionality Reduction and Signal
	         Recovery: A Geometric Perspective},
	author = {Baraniuk, Richard G. and Cevher, Volkan and Wakin, Michael B.},
	journal = {Proceedings of the IEEE},
	volume = {98},
	pages = {959--971},
	year = {2010},
}

@inproceedings{li2010recommandation,
	title = {A contextual-bandit approach to personalized news article
	         recommendation},
	author = {Li, Lihong and Chu, Wei and Langford, John and Schapire, Robert E.
	          },
	booktitle = {Proceedings of the 19th International Conference on World Wide
	             Web},
	pages = {661--670},
	year = {2010},
}

@inproceedings{li2017provablyGLB,
	title = {Provably Optimal Algorithms for Generalized Linear Contextual
	         Bandits},
	author = {Lihong Li and Yu Lu and Dengyong Zhou},
	booktitle = {Proceedings of the 34th International Conference on Machine
	             Learning},
	pages = {2071--2080},
	volume = {70},
	year = {2017},
}

@inproceedings{luo2016ONSsketchig,
	title = {Efficient Second Order Online Learning by Sketching},
	booktitle = {Advances in Neural Information Processing Systems},
	author = {Luo, Haipeng and Agarwal, Alekh and Cesa-Bianchi, Nicolò and
	          Langford, John},
	year = {2016},
	volume = {29},
}

@article{Hazan2007ons,
	title = {Logarithmic Regret Algorithms for Online Convex Optimization},
	journal = {Machine Learning},
	author = {Hazan, Elad and Agarwal, Amit and Kale, Satyen},
	year = {2007},
	volume = {69},
	pages = {169--192},
}

@article{wan2022adafd,
	title = {Efficient Adaptive Online Learning via Frequent Directions},
	journal = {{IEEE} Transactions on Pattern Analysis and Machine Intelligence},
	author = {Wan, Yuanyu and Zhang, Lijun},
	year = {2022},
	pages = {6910--6923},
	volume = {44},
	number = {10},
}

@inproceedings{yang2025dimensionfree,
	title = {Dimension-Free Adaptive Subgradient Methods with Frequent
	         Directions},
	author = {Sifan Yang and Yuanyu Wan and Peijia Li and Yibo Wang and Xiao
	          Zhang and Zhewei Wei and Lijun Zhang },
	booktitle = {Proceedings of The 42nd International Conference on Machine
	             Learning},
	pages = {},
	year = {2025},
}

@article{luo2019rfd,
	title = {Robust Frequent Directions with Application in Online Learning},
	author = {Luo Luo and Cheng Chen and Zhihua Zhang and Wu-Jun Li and Tong
	          Zhang},
	journal = {Journal of Machine Learning Research},
	year = {2019},
	volume = {20},
	number = {45},
	pages = {1--41},
}

@article{lecun1998gradient,
	title = {Gradient-based learning applied to document recognition},
	author = {LeCun, Yann and Bottou, L{\'e}on and Bengio, Yoshua and Haffner,
	          Patrick},
	journal = {Proceedings of the IEEE},
	volume = {86},
	number = {11},
	pages = {2278--2324},
	year = {1998},
}

@inproceedings{Lee2024r2cs,
	title = {Improved Regret Bounds of (Multinomial) Logistic Bandits via
	         Regret-to-Confidence-Set Conversion},
	author = {Lee, Junghyun and Yun, Se-Young and Jun, Kwang-Sung},
	booktitle = {Proceedings of The 27th International Conference on Artificial
	             Intelligence and Statistics},
	pages = {4474--4482},
	volume = {238},
	year = {2024},
}

@misc{boudart2025enjoyNonLinear,
	title = {Enjoying Non-linearity in Multinomial Logistic Bandits: A
	         Minimax-Optimal Algorithm},
	author = {Pierre Boudart and Pierre Gaillard and Alessandro Rudi},
	year = {2026},
	howpublished = {arXiv:2507.05306},
	eprint = {2507.05306},
	archivePrefix = {arXiv},
	primaryClass = {stat.ML},
}

@inproceedings{lee2025ImprovedConfidBound4MLogB,
	title = {Improved Online Confidence Bounds for Multinomial Logistic Bandits},
	author = {Lee, Joongkyu and Oh, Min-Hwan},
	booktitle = {Proceedings of the 42nd International Conference on Machine
	             Learning},
	pages = {33576--33615},
	volume = {267},
	year = {2025},
}

@inproceedings{ding2021sgdts4glb,
	title = { An Efficient Algorithm For Generalized Linear Bandit: Online
	         Stochastic Gradient Descent and Thompson Sampling },
	author = {Ding, Qin and Hsieh, Cho-Jui and Sharpnack, James},
	booktitle = {Proceedings of The 24th International Conference on Artificial
	             Intelligence and Statistics},
	pages = {1585--1593},
	volume = {130},
	year = {2021},
}


\appendix

\section{Properties of Algorithm~\ref{algo:adpfd}} \label{sec:rfd-properties}

In this section, we provide some useful lemmas of Algorithm~\ref{algo:adpfd}, and corresponding necessary proofs.

\begin{lemma} \label{lem:fd-ineq-telescope}
\[
	V_{t+1} \Lambda_{t+1} V_{t+1}^\top
	\preceq
	V_t \Lambda_t V_t^\top + \nabla^2 \ell_t(W_{t+1})
	\preceq
	V_{t+1} \Lambda_{t+1} V_{t+1}^\top + (\rho_{1:t} - \rho_{1:t-1}) I_{Kd}
.\]
\begin{proof}
By the eigen-decomposition and the definition of $B_t$ in line 1 and line 2 of Algorithm \ref{algo:adpfd},
we have
\[
	[V_t \Lambda_t^{\frac{1}{2}}, B_t] [V_t \Lambda_t^{\frac{1}{2}}, B_t]^\top
	= V_t \Lambda_t V_t^\top + B_t B_t^\top
	= V_t \Lambda_t V_t^\top + P D P^\top \otimes \x_t \x_t^\top
	= V_t \Lambda_t V_t^\top + \nabla^2 \ell_t(W_{t+1})
. \]
On the other hand, by the definition of $\bar{U}, \bar{\Sigma}, \bar{V}$ in line 3, we have
\[
	\bar{V} \bar{\Sigma}^2 \bar{V}^\top
	= [V_t \Lambda_t^{\frac{1}{2}}, B_t] [V_t \Lambda_t^{\frac{1}{2}}, B_t]^\top
. \]
Thus
\[
	\bar{V} \bar{\Sigma}^2 \bar{V}^\top
	= V_t \Lambda_t V_t^\top + \nabla^2 \ell_t(W_{t+1})
. \]
Let $\bar{\sigma}_i$ be the $i$-th largest singular value of $\bar{\Sigma}$.
Then we have
\[
	\bar{\sigma}_1^2 \ge \dots \ge \bar{\sigma}_m^2 \ge \bar{\sigma}_{m+1}^2 \ge \bar{\sigma}_{m+2}^2 \ge \dots \ge \bar{\sigma}_{m+K}^2 \ge 0
. \]
Let $\tilde{\Sigma} = \diag(\bar{\sigma}_{m+1}, \dots, \bar{\sigma}_{m+1}, \bar{\sigma}_{m+2}, \bar{\sigma}_{m+3}, \dots, \bar{\sigma}_{m+K}) \in \R^{(m+K)\times(m+K)}$.
By line 5 and line 6, we have
\[
	V_{t+1} \Lambda_{t+1} V_{t+1}^\top + \bar{V} \tilde{\Sigma}^2 \bar{V}^\top
	= \bar{V} \bar{\Sigma}^2 \bar{V}^\top
. \]
To conclude
\[
	V_{t+1} \Lambda_{t+1} V_{t+1}^\top + \bar{V} \tilde{\Sigma}^2 \bar{V}^\top
	= V_t \Lambda_t V_t^\top + \nabla^2 \ell_t(W_{t+1})
. \]
By $\bar{V} \tilde{\Sigma}^2 \bar{V}^\top \succeq 0$, we have the first inequality of the claim:
\[
	V_{t+1} \Lambda_{t+1} V_{t+1}^\top
	\preceq
	V_t \Lambda_t V_t^\top + \nabla^2 \ell_t(W_{t+1})
. \]
And by $\bar{V} \tilde{\Sigma}^2 \bar{V}^\top \preceq \bar{\sigma}_{m+1}^2 I_{Kd}$, we have the second inequality of the claim:
\[
	V_t \Lambda_t V_t^\top + \nabla^2 \ell_t(W_{t+1})
	\preceq
	V_{t+1} \Lambda_{t+1} V_{t+1} + \bar{\sigma}_{m+1}^2 I_{Kd}
	=
	V_{t+1} \Lambda_{t+1} V_{t+1} + (\rho_{1:t} - \rho_{1:t-1}) I_{Kd}
. \]
\end{proof}
\end{lemma}

\begin{lemma}[Property 1, 2 of \citet{ghashami2016fd}] \label{lem:fd-sketch-relation}
Let $M_t = \sum_{s=1}^{t-1} \nabla^2 \ell_s(W_{s+1})$ be the cumulative Hessian, then
\[
	V_t \Lambda_t V_t^\top
	\preceq
	M_t
	\preceq
	V_t \Lambda_t V_t^\top + \rho_{1:t-1} I_{Kd}
. \]
\begin{proof}
By the first inequality of Lemma~\ref{lem:fd-ineq-telescope}, we have
\[
	V_{s+1} \Lambda_{s+1} V_{s+1}^\top - V_s \Lambda_s V_s^\top
	\preceq \nabla^2 \ell_s(W_{s+1})
. \]
Summing over $s$ from $1$ to $t-1$, we have
\[
	V_t \Lambda_t V_t^\top
	\preceq \sum_{s=1}^{t-1} \nabla^2 \ell_s(W_{s+1}) + V_1 \Lambda_1 V_1^\top
	= M_t
. \]
By the second inequality of Lemma~\ref{lem:fd-ineq-telescope}, we have
\[
	\nabla^2 \ell_s(W_{s+1})
	\preceq
	V_{s+1} \Lambda_{s+1} V_{s+1} - V_s \Lambda_s V_s^\top + (\rho_{1:s} - \rho_{1:s-1}) I_{Kd}
. \]
Summing over $s$ from $1$ to $t-1$, we have
\[
	M_t
	\preceq
	V_t \Lambda_t V_t^\top + \sum_{s=1}^{t-1} (\rho_{1:s} - \rho_{1:s-1}) I_{Kd}
	=
	V_t \Lambda_t V_t^\top + \rho_{1:t-1} I_{Kd}
. \]
\end{proof}
\end{lemma}

\begin{lemma} \label{lem:fd-reg-sketch-relation}
Let $M_t = \sum_{s=1}^{t-1} \nabla^2 \ell_s(W_{s+1})$ be the cumulative Hessian, then for any $\lambda > 0$,
\[
	V_t \Lambda_t V_t^\top + \lambda I_{Kd}
	\preceq
	M_t + \lambda I_{Kd}
	\preceq
	\left(1 + \frac{\rho_{1:t-1}}{\lambda}\right)
	\left( V_t \Lambda_t V_t^\top + \lambda I_{Kd} \right)
\]
\begin{proof}
We follow the main sketch of lemma 9 of \citet{kuzborskij2019efficient}.
By the second inequality of Lemma~\ref{lem:fd-sketch-relation},
the first inequality of the claim can be proved easily.
\par 
For the second inequality, let $(\lambda_i, \bv_i)$ be the $i$-th full eigen-pair of $V_t \Lambda_t V_t^\top$, then we have the decomposition \[
	V_t \Lambda_t V_t^\top + \lambda I_{Kd} = \sum_{i=1}^{Kd} (\lambda_i + \lambda) \bv_i \bv_i^\top
.\] And by the first inequality of property \ref{lem:fd-sketch-relation}
\[
	M_t
	\preceq
	S_t^\top S_t + \rho_{1:t-1} I_{Kd}
	=
	\sum_{i=1}^{Kd} (\lambda_i + \lambda + \rho_{1:t-1}) \bv_i \bv_i^\top
.\]
Therefore
\begin{align*}
	&~
	(V_t \Lambda_t V_t^\top + \lambda I_d)^{-\frac{1}{2}} (M_t + \lambda I_d) (V_t \Lambda_t V_t^\top + \lambda I_d)^{-\frac{1}{2}}
	\\
	\preceq &~
	(V_t \Lambda_t V_t^\top + \lambda I_d)^{-\frac{1}{2}} (V_t \Lambda_t V_t^\top + (\rho_{1:t-1} + \lambda) I_d) (V_t \Lambda_t V_t^\top + \lambda I_d)^{-\frac{1}{2}}
	\\
	\preceq &~
	\sum_{t=1}^{Kd} \frac{\lambda_i + \lambda + \rho_{1:t-1}}{\lambda_i + \lambda} \bv_i \bv_i^\top
	\\
	\preceq &~
	\sum_{t=1}^{Kd} \frac{\lambda + \rho_{1:t-1}}{\lambda} \bv_i \bv_i^\top
	= \frac{\lambda + \rho_{1:t-1}}{\lambda} I
\end{align*}
Left and right multiply $(V_t \Lambda_t V_t^\top + \lambda I_{Kd})^{\frac{1}{2}}$ to get \[
	M_t + \lambda I_{Kd}
	\preceq
	\left(1 + \frac{\rho_{1:t-1}}{\lambda}\right) (V_t \Lambda_t V_t^\top + \lambda I_{Kd})
\]
\end{proof}
\end{lemma}

\begin{lemma}[Lemma 8 of \citet{kuzborskij2019efficient}] \label{lem:fd-det-origin}
Let $M_t = \sum_{s=1}^{t-1} \nabla^2 \ell_s(W_{s+1})$ be the cumulative Hessian, then for any $\lambda > 0, t > 0$,
\[
	\frac{\left| M_t + \lambda I_{Kd} \right|}{\left| \lambda I_{Kd} \right|}
	\le
	\left(1 + \frac{\rho_{1:t}}{\lambda}\right)^{Kd}
	\left( 1 + \frac{\tr(M_t)}{m\lambda} \right)^m
. \]
\begin{proof}
Let $(\mu_i, \bu_i)$ be the full eigen-pair of $M_t$,
and $(\lambda_i, \bv_i)$ be the full eigen-pair of $V_t \Lambda_t V_t$,
then we have the decomposition
\begin{align*}
	M_t + \lambda I_{Kd} &= \sum_{i=1}^{Kd} (\mu_i + \lambda) \bu_i \bu_i^\top
	\\
	V_t \Lambda_t V_t &= \sum_{i=1}^{Kd} \lambda_i \bv_i \bv_i^\top
\end{align*}
Therefore,
\begin{align*}
\left| M_t + \lambda I_d \right|
&= \prod_{i=1}^{Kd} (\mu_i + \lambda)
\\&~\text{(by the first inequality of Lemma~\ref{lem:fd-sketch-relation})}
\\&\le
\prod_{i=1}^{Kd} \left( \lambda_i + \rho_{1:t-1} + \lambda \right)
\\&~\text{(the rank of $V_t \Lambda_t V_t$ is at most $m$)}
\\&=
(\rho_{1:t-1} + \lambda)^{Kd - m}
\prod_{i=1}^{m} \left( \lambda_i + \rho_{1:t-1} + \lambda \right)
\\&~\text{(by AM-GM inequality)}
\\&\le
(\rho_{1:t-1} + \lambda)^{Kd - m}
\left(\sum_{i=1}^{m} \frac{\lambda_i}{m} + \rho_{1:t-1} + \lambda\right)^m
\\&=
(\rho_{1:t-1} + \lambda)^{Kd}
\left( \frac{\sum_{i=1}^{m} \lambda_i}{m(\rho_{1:t-1} + \lambda)} + 1\right)^m
\\&=
(\rho_{1:t-1} + \lambda)^{Kd}
\left( \frac{\tr(V_t \Lambda_t V_t)}{m(\rho_{1:t-1} + \lambda)} + 1\right)^m
\\&~\text{(by the second inequality of property \ref{lem:fd-sketch-relation})}
\\&\le
(\rho_{1:t-1} + \lambda)^{Kd}
\left( 1 + \frac{\tr(M_t)}{m(\rho_{1:t-1} + \lambda)} \right)^m
\end{align*}
And \[
	\left| \lambda I_{Kd} \right| = \lambda^{Kd}
, \] therefore \[
	\frac{\left| M_t + \lambda I_{Kd} \right|}{\left| \lambda I_{Kd} \right|}
	\le
	\left(1 + \frac{\rho_{1:t-1}}{\lambda}\right)^{Kd}
	\left( 1 + \frac{\tr(M_t)}{m(\rho_{1:t-1} + \lambda)} \right)^m
	\le
	\left(1 + \frac{\rho_{1:t}}{\lambda}\right)^{Kd}
	\left( 1 + \frac{\tr(M_t)}{m\lambda} \right)^m
. \]
\end{proof}
\end{lemma}

\newpage
\section{Theorem in section \ref{sec:fd-oful-mlogb}} \label{sec:proof-oful-mlogb}

In this section, we will give the proof details of Theorem \ref{thm:confidence-region}, Theorem \ref{thm:regret-K-ge-2} and Proposition \ref{prop:regret-k-eq-1}.

\subsection{Preliminary} \label{ssec:apd-preliminary}
In this section, we explain some notations and properties used in the proof.
The multinomial logistic function $\sig: \R^K \mapsto (0, 1)^{K+1}$ is defined as \[
	\sig(\z) = \frac{1}{1 + \sum_{k=1}^{K} \exp(x_k)} \left[1, \exp(x_1), \exp(x_2), \dots, \exp(x_K)\right]^\top
. \] However, since the $0$-th component of $\rwdv$ is $\rwd_0 = 0$, it does not contribute to the reward and the regret.
Namely, $\sum_{k=0}^{K} \rho_k \sigma_k(W\x) = \sum_{k=1}^K \rho_k \sigma_k(W\x)$.
Thus we ignore the $0$-th component of $\sig(\z)$, which is refined as \[
	\sig(\z) = \frac{1}{1 + \sum_{k=1}^{K} \exp(x_k)} \left[\exp(x_1), \exp(x_2), \dots, \exp(x_K)\right]^\top \in (0, 1)^{K}
. \]
\par 
For a special case where $K = 1$, the reward vector is defined as $\rwdv = [0, 1]^\top$.
Therefore the expected reward for given action $\x \in \A$ is \[
	\rwdv^\top \sig(\ws^\top \x) = 0 \cdot \sigma_0(\ws^\top \x) + 1 \cdot \sigma_1(\ws^\top \x) = \sigma_1(\ws^\top \x)
, \] where $\sigma_1(z) = \exp(z) / (1 + \exp(z))$.
And the regret analysis can be reduced to the binary logistic function $\sigma_1(z)$ without $\rwdv$ involved \citep{zhang2016logistic}.
\par 
The first order derivative of $\sig(\z)$ is \[
	\nabla \sig(\z) = \diag{\sig(\z)} - \sig(\z) \sig(\z)^\top
. \] And the second order derivative of $\rwdv^\top\sig(\z)$ is \[
	D^2 \sig(\z)[\rwdv]
	= \sum_{k=1}^K \rwd_k \sigma_k(\z) \left(2 \sig(\z) \sig(\z)^\top - \diag(\sig(\z)) - \e_k \sig(\z)^\top - \sig(\z) \e_k^\top + \e_k \e_k^\top\right)
, \] which is bounded by \begin{equation} \label{ineq:second-derivative-sigma-bound}
	- \|\rwdv\|_1 I \preceq D^2 \sig(\z)[\rwdv] \preceq \frac{1}{2} \|\rwdv\|_1 I
. \end{equation}

Let $\x_t$ be the action selected at round $t$ and $y_t \in \{0, 1, \dots, K\}$ be the observed outcome and $\sigma_k(\z) = [\sig(\z)]_k$ be the $k$-th component of $\sig(\z)$.
We denote $\y_t = [\I\{y_t = 1\}, \I\{y_t = 2\}, \dots, \I\{y_t = K\}]^\top$.
The logistic loss function $\ell_t: \R^{K \times d} \mapsto \R$ is defined as \[
	\ell_t(W) = - \ln \sigma_{y_t}(W \x_t)
, \] and the first and second order derivative of $\ell_t(W)$ are \begin{align*}
	\nabla \ell_t(W)
	&= (\sig(W \x_t) - \y_t) \otimes \x_t
	\\
	\nabla^2 \ell_t(W)
	&= \nabla \sig(W\x_t) \otimes \x_t \x_t^\top \\
	&= (\diag(\sig(W\x_t)) - \sig(W\x_t)\sig(W\x_t)^\top) \otimes \x_t \x_t^\top
\end{align*}
And this logistic function has $\sqrt6$-self-concordant-like property \citep[Lemma 4]{tran-dinh2015composite-self-concordant} and exp-concave-like property \citep[Lemma 4]{jezequel2021mixiability}, which is \begin{equation} \label{ineq:ell-self-concordant-like}
	\left| \left< D^3\ell_t(W)[\bu_1] \bu_2, \bu_3 \right> \right|
	\le \sqrt6 \left\| \bu_1 \right\|_2 \left\| \bu_2 \right\|_{\nabla^2 \ell_t(W)} \left\| \bu_3 \right\|_{\nabla^2 \ell_t(W)}
, \end{equation} and \begin{equation} \label{ineq:exp-concave-like}
	\ell_t(W_1) - \ell_t(W_2)
	\le \left< \nabla \ell_t(W_1), W_1 - W_2 \right> - \frac{1}{\ln(K+1) + 2 (SD + 1)} \left\| W_1 - W_2 \right\|_{\nabla^2 \ell_t(W_1)}^2
. \end{equation}

\subsection{Proof for Theorems} \label{ssec:proofthm}
We denote $\tail{t}{\lambda} = 1 + \rho_{1:t}/\lambda$ as the sketching error factor in the following proofs.
\paragraph{Proof for Theorem \ref{thm:confidence-region}}
\begin{proof}
We follow the proof of \citet[Theorem 3]{zhang2023mlb}, and emphasize the difference where FD are applied.
To prove the theorem \ref{thm:confidence-region} using the sketched technique,
the FD Lemma \ref{lem:fd-ineq-telescope} and \ref{lem:fd-det-origin} is pivotal.
We fix the step size $\eta = \frac{1}{2} (\ln(K+1) + 2 (SD + 1))$
\par 
In the OMD analysis, the surrogate loss function \[
	\ellt_t(W) := \left< \nabla \ell_t(\Wv_t), \Wv - \Wv_t \right> + \frac{1}{2} \left\| \Wv - \Wv_t \right\|_{\nabla^2 \ell_t(W_t)}^2
\] is a second order approximation of the loss function $\ell_t(W)$ on the point $W_t$ instead of the linear approximation $\left< \nabla \ell_t(\Wv_t), \Wv - \Wv_t \right>$ in traditional ONS. And the update rule is equivalent to \[
	W_{t+1} := \argmin_{W \in \W}
		\eta \ellt_t(W)
		+ \frac{1}{2} \left\| \Wv - \Wv_t \right\|_{Z_t}^2
. \] By the first order optimal condition, we have \[
	\forall \bu \in \W, \left< \eta \nabla \ellt_t(W_t) + Z_t(\Wv_{t+1} - \Wv_t), \bu - \Wv_{t+1} \right> \ge 0
. \] Let $\bu = \Wvs$, we have \[
	\eta \left< \nabla \ellt_t(\Wv_t), \Wv_{t+1} - \Wvs \right>
	\le \frac{1}{2} \left\| \Wvs - \Wv_t  \right\|_{Z_t}^2
		- \frac{1}{2} \left\| \Wvs - \Wv_{t+1}  \right\|_{Z_t}^2
		- \frac{1}{2} \left\| \Wv_t - \Wv_{t+1}  \right\|_{Z_t}^2
. \] And by \eqref{ineq:exp-concave-like} \[
	\eta(\ell_t(W_{t+1}) - \ell_t(\Ws))
	\le \eta \left< \nabla \ell_t(\Wv_{t+1}), \Wv_{t+1} - \Wvs \right> - \frac{1}{2} \left\| W_{t+1} - \Ws \right\|_{\nabla^2 \ell_t(W_{t+1})}^2
. \] Add the two inequalities together and rearrange to get \begin{align} \label{ineq:ps-cr-range}
	&~
	\frac{1}{2} \left\| \Wvs - \Wv_{t+1}  \right\|_{Z_t + \nabla^2 \ell_t(W_{t+1})}^2
	- \frac{1}{2} \left\| \Wvs - \Wv_t  \right\|_{Z_t}^2
	+ \frac{1}{2} \left\| \Wv_t - \Wv_{t+1}  \right\|_{Z_t}^2
	\notag\\ \le&~
	\eta (\ell_t(\Ws) - \ell_t(W_{t+1}))
	+ \eta \left< \nabla \ell_t(\Wv_{t+1}) - \nabla \ellt_t(\Wv_{t+1}), \Wv_{t+1} - \Wvs \right>
\end{align}
In left hand side \begin{align*}
	\frac{1}{2} \left\| \Wvs - \Wv_{t+1} \right\|_{Z_t + \nabla^2 \ell_t(W_{t+1})}^2
	&= \frac{1}{2} \left\| \Wvs - \Wv_{t+1} \right\|_{\lambda I + V_t \Lambda_t V_t^\top + \nabla^2 \ell_t(W_{t+1})}^2
	\\
	&\ge \frac{1}{2} \left\| \Wvs - \Wv_{t+1} \right\|_{\lambda I + V_{t+1} \Lambda_{t+1} V_{t+1}^\top}^2
	\\
	&= \frac{1}{2} \left\| \Wvs - \Wv_{t+1} \right\|_{Z_{t+1}}^2
\end{align*} where the inequality is by FD Lemma \ref{lem:fd-ineq-telescope}.
And in the right hand side, the second term $\langle \nabla \ell_t(\Wv_{t+1}) - \nabla \ellt_t(\Wv_{t+1}), \Wv_{t+1} - \Wvs \rangle$
is the first order Taylor remainder of the loss function $f(W) = \left< \nabla \ell_t(\Wv), \Wv_{t+1} - \Wvs \right>$ between the point $\Wv_{t+1}$ and $\Wv_t$,
and thus can be bounded by \eqref{ineq:ell-self-concordant-like}.
\begin{align*}
	\left< \nabla \ell_t(\Wv_{t+1}) - \nabla \ellt_t(\Wv_{t+1}), \Wv_{t+1} - \Wvs \right>
	&=
	\left< \frac{1}{2!} D^3\ell_t(\xi) \left[\Wv_t - \Wv_{t-1}, \Wv_t - \Wv_{t-1}\right], \Wv_t - \Wvs \right>
	\\
	&=
	\frac{1}{2} D^3\ell_t(\xi) \left[\Wv_t - \Wv_{t-1}, \Wv_t - \Wv_{t-1}, \Wv_t - \Wvs\right]
	\\
	&=
	\frac{1}{2} D^3\ell_t(\xi) \left[\Wv_t - \Wvs, \Wv_t - \Wv_{t-1}, \Wv_t - \Wv_{t-1}\right]
	\\
	&\le
	\frac{\sqrt6}{2} \left\| \Wv_t - \Wvs \right\|_2 \left\| \Wv_t - \Wv_{t-1} \right\|_{\nabla^2 \ell_t(\xi)}^2
	\\
	&\le
	\sqrt6 L S \left\| \Wv_{t+1} - \Wv_t \right\|_{I\otimes\x_t\x_t^\top}^2
\end{align*}
Plug the above inequality into \eqref{ineq:ps-cr-range} and sum over $t$ to get
\begin{align} \label{ineq:ps-cr-range-sum}
	\left\| \Wv_{t+1} - \Ws \right\|_{Z_{t+1}}^2 - \left\| \Wv_1 - \Ws \right\|_{Z_1}^2
	\le &~
	2 \eta \sum_{s=1}^{t} \ell_s(\Ws) - \ell_s(W_{s+1})
	+ 2 \sqrt6 \eta L S \sum_{s=1}^{t} \left\| \Wv_{s+1} - \Wv_s \right\|_{I\otimes\x_s\x_s^\top}^2
	\notag\\ &
	- \sum_{s=1}^{t} \left\| \Wv_{s+1} - \Wv_s \right\|_{Z_s}^2
. \end{align} $\sum_{s=1}^{t} \ell_s(\Ws) - \ell_s(W_{s+1})$ can be bounded by Lemma \ref{lem:cr-ell-regret} for any martingale sequence $\{ (W_t, \x_t) \,|\, t \in \N \}$.
\begin{align*}
	\sum_{s=1}^{t} \ell_s(\Ws) - \ell_s(W_{s+1})
	\le &~
	\left(D S + \ln(K+1) + 2 \ln ((K+1) t + 2)\right) \left(\frac{5}{4}  + 4 \ln \frac{\sqrt{1 + 2 t}}{\delta}\right)
	+
	2 + \ln (K+1)
	\\ &
	+
	\frac{\sqrt6}{24 Kd} \lambda \tail{t}{\lambda} \left( Kd \ln \tail{t}{\lambda/2} + m \ln (1 + 2\frac{D^2}{m \lambda}t) \right)
	+ \frac{12 Kd}{\lambda} \sum_{s=1}^{t} \left\| \Wv_{s+1} - \Wv_s \right\|_{Z_s}^2
\end{align*}
We apply the lemma \ref{lem:cr-ell-regret} to \eqref{ineq:ps-cr-range-sum} ,
and tune the regularized parameter $\lambda \ge \Max{2 L D^2, 2 \eta \Max{2 \sqrt6 D^2 S L, 24 Kd}}$ to get the last three terms cancel out \begin{align*}
	&~
	\frac{24 \eta Kd}{\lambda} \sum_{s=1}^{t} \left\| \Wv_s - \Wv_{s+1} \right\|_{Z_s}^2
	+ 2 \sqrt6 \eta L S \sum_{s=1}^{t} \left\| \Wv_{s+1} - \Wv_s \right\|_{I \otimes \x_t \x_t^\top}^2
	- \sum_{s=1}^{t} \left\| \Wv_{s+1} - \Wv_s \right\|_{Z_s}^2
	\\
	\le&~
	\frac{24 \eta Kd}{\lambda} \sum_{s=1}^{t} \left\| \Wv_s - \Wv_{s+1} \right\|_{Z_s}^2
	+ 2 \sqrt6 \eta L S D^2 \sum_{s=1}^{t} \left\| \Wv_{s+1} - \Wv_s \right\|_2^2
	- \sum_{s=1}^{t} \left\| \Wv_{s+1} - \Wv_s \right\|_{Z_s}^2
	\\
	\le &~
	0
. \end{align*}
Therefore, with $\eta = \O(\ln K)$, we have \[
	\left\| \Wv_{t+1} - \Ws \right\|_{Z_{t+1}}^2
	\le
	\beta_{t+1}(\delta)
, \] where
\begin{align}
	\beta_{t+1}(\delta)
	=&~
	\lambda (\left\| W_1 \right\|_F + S)^2
	\notag\\&
	+ 2 \eta \left(D S + \ln(K+1) + 2 \ln ((K+1) t + 2)\right) \left(\frac{5}{4}  + 4 \ln \frac{\sqrt{1 + 2 t}}{\delta}\right)
	+ 2 \eta (2 + \ln (K+1))
	\notag\\&
	+ \frac{\sqrt6 \eta}{12 Kd} \lambda \tail{t}{\lambda} \left(K d \ln \tail{t}{\lambda/2} + m \ln (1 + 2\frac{D^2}{m \lambda} t)\right)
	\notag\\&
	+ \sum_{s=1}^{t} 2\sqrt6 \eta L S \left\| (W_{s+1} - W_s) \x_s \right\|_2^2 + (\frac{24 \eta K d}{\lambda} - 1) \left\| \Wv_{s+1} - \Wv_s \right\|_{Z_s}^2
	\label{eq:cr-beta}
\end{align}
And
\[
	\beta_{t+1}(\delta)
	=
	\O\left(\ln K (\ln K + \ln t) \tail{t}{\lambda} \left(K d \ln \tail{t} + m \ln t + \ln \frac{1}{\delta}\right)\right)
. \]
\end{proof}

\paragraph{Proof for Theorem \ref{thm:regret-K-ge-2}}
\begin{proof}
The actual expected reward $\rwdv^\top \sig(\Ws \x)$ and the estimated reward $\rwdv^\top \sig(W_t \x)$ has the following difference:
\begin{lemma}
\label{lem:diff-rwd-upper-bound}
For any $\delta \in (0, 1)$, with probability at least $1 - \delta$, we have \[
	\forall t > 0.~
	\left| \rwdv^\top \sig(\Ws \x) - \rwdv^\top \sig(W_t \x) \right|
	\le
	b_t(\x, \delta)
	=
	\dfst_t(\x, \delta) + \dsnd_t(\x, \delta)
, \] where \[
	\dfst_t(\x, \delta)
	=
	\sqrt{\beta_t(\delta)} \left\| (\nabla \sig(W_t \x) \otimes \x) \rwdv \right\|_{Z_t^{-1}}
, \] and \[
	\dsnd_t(\x, \delta)
	=
	\frac{\|\rwdv\|_1 }{2} \beta_t(\delta) \left\| Z_t^{-\frac{1}{2}} (I \otimes \x) \right\|_{2,2}^2
. \]
\end{lemma}
The proof of the above lemma is similar to \citet[C.2.1]{zhang2023mlb}, and the main difference is the replacement of $H_t$ with $Z_t$.
We delay its proof to Appendix \ref{proof:lem-diff-rwd-upper-bound}.
\begin{corollary}[Proposition 1 of \citet{zhang2023mlb}]
For any $\delta \in (0, 1)$, with probability at least $1 - \delta$, we have \[
	\forall t > 0.~
	\rwdv^\top \sig(\Ws \x)
	\le
	R_t(\x, \delta)
	=
	\rwdv^\top \sig(W_t \x)
	+ b_t(\x, \delta)
. \]
\end{corollary}
The regret is defined as the expected reward of the best action minus the expected reward of the action taken by the algorithm:
\begin{align*}
\reg(T)
= &~
\sum_{t=1}^{T} \left( \rwdv^\top \sig(\Ws \xs) - \rwdv^\top \sig(\Ws \x_t) \right)
\\
= &~
\sum_{t=1}^{T}
\underbrace{
	\rwdv^\top \sig(\Ws \xs) - \rwdv^\top \sig(W_t \xs)
}_{\texttt{term a}}
+
\underbrace{
	\rwdv^\top \sig(W_t \xs) - \rwdv^\top \sig(W_t \x_t)
}_{\texttt{term b}}
\\ &~
+
\underbrace{
	\rwdv^\top \sig(W_t \x_t)- \rwdv^\top \sig(\Ws \x_t)
}_{\texttt{term c}}
\end{align*}
The \texttt{term a} and \texttt{term c} is bounded \[
	\texttt{term a} \le b_t(\xs, \delta)
	,
	\texttt{term c} \le b_t(\x_t, \delta)
, \] by the lemma \ref{lem:diff-rwd-upper-bound}.
The \texttt{term b} is bounded \[
	R_t(\xs, \delta) \le R_t(\x_t, \delta)
	\implies
	\texttt{term b} \le b_t(\x_t, \delta) - b_t(\xs, \delta)
, \] by the decision rule $\x_t := \argmax_{\x \in \D_t} R_t(\x, \delta)$.
Thus the regret is bounded by $2 \sum_{t=1}^{T} b_t(\x_t, \delta)$. \begin{equation} \label{ineq:reg-less-delta}
	\reg(T)
	\le
	2 \sum_{t=1}^{T} b_t(\x_t, \delta)
	=
	2 \sum_{t=1}^{T} \dfst_t(\x_t, \delta) + 2 \sum_{t=1}^{T} \dsnd_t(\x_t, \delta)
. \end{equation}
And we substitute $H_t$ to $Z_t$ to get the lemma holds in the proposed algorithm.
For the term $\sum_{t=1}^{T} \dfst_{t-1}(\x_t, \delta)$,
\begin{align*}
&~
\sum_{t=1}^{T} \dfst_t(\x_t, \delta)
=
\sum_{t=1}^{T}
\sqrt{\beta_t(\delta)}
\left\| (\nabla \sig(W_t \x_t) \otimes \x_t) \rwdv \right\|_{Z_t^{-1}}
\\
\le &~
\sqrt{\beta_T(\delta)} \sum_{t=1}^{T}
\left\| (\nabla \sig(W_{t+1} \x_t) \otimes \x_t) \rwdv \right\|_{Z_t^{-1}}
+
\left\| ((\nabla \sig(W_t \x_t) - \nabla \sig(W_{t+1} \x_t)) \otimes \x_t) \rwdv \right\|_{Z_t^{-1}}
\\
\le &~
\sqrt{\beta_T(\delta)}
\sum_{t=1}^{T}
\left\| \nabla \sig(W_{t+1} \x_t)^{\frac{1}{2}} \rwdv \right\|_2
\left\| \nabla \sig(W_{t+1} \x_t)^{\frac{1}{2}} \otimes \x_t\right\|_{2,Z_t^{-1}}
\\ &
+
\frac{1}{2}
\max_{\xi_1} \left\| Z_t^{-\frac{1}{2}} (D^2 \sig(\xi_1)[\rwdv] \otimes \x_t \x_t^\top) Z_t^{-\frac{1}{2}} \right\|_{2,2}
\left\|\Wv_t - \Wv_{t+1}\right\|_{Z_t}
\end{align*}
where the second inequality is by Mean Value Theorem of vector function in form of finite increment.
\par 
The following terms is bounded: $\left\| \nabla \sig(W_{t+1} \x_t)^{\frac{1}{2}} \rwdv \right\|_2 \le R$,
and $\left\|\Wv_t - \Wv_{t+1}\right\|_{Z_t} \le \frac{2 \eta D}{\sqrt{\lambda}}$ by lemma \ref{lem:omd-wt-wt-1-bound}.
And moreover
\begin{align*}
\left\| \nabla \sig(W_{t+1} \x_t)^{\frac{1}{2}} \otimes \x_t\right\|_{2,Z_t^{-1}}
&\le \sqrt{\tr((\nabla \sig(W_{t+1} \x_t) \otimes \x_t \x_t^\top) Z_t^{-1})}
= \sqrt{\tr(\nabla^2 \ell_t(W_{t+1}) Z_t^{-1})}
\\
\max_{\xi_1} \left\| Z_t^{-\frac{1}{2}} (D^2 \sig(\xi_1)[\rwdv] \otimes \x_t \x_t^\top) Z_t^{-\frac{1}{2}} \right\|_{2,2}
&\le \|\rwdv\|_1 \lambda_{\max}( (I \otimes \x_t^\top) Z_t^{-1} (I \otimes \x_t) )
\le \|\rwdv\|_1 \tr\left(Z_t^{-1} (I \otimes \x_t \x_t^\top)\right)
\end{align*}
in which the second inequality is by $- \|\rwdv\|_1 I \preceq D^2 \sig(\xi_1)[\rwdv] \preceq \frac{1}{2} \|\rwdv\|_1 I$.
Thus by GM-AM inequality, we have
\begin{equation} \label{ineq:regret-dfst-bound}
\sum_{t=1}^{T} \dfst_t(\x_t, \delta)
\le
R \sqrt{\beta_T(\delta)} \sqrt{T} \sqrt{\sum_{t=1}^{T} \tr\left( Z_t^{-1} \nabla^2 \ell_t(W_{t+1}) \right)}
+
\frac{2 \eta D}{\sqrt{\lambda}} \|\rwdv\|_1
\sqrt{\beta_{T-1}(\delta)}
\sum_{t=1}^{T} \tr\left( Z_t^{-1} (I \otimes \x_t \x_t^\top) \right)
\end{equation}
\par 
And for the term $\dsnd_t(\x_t, \delta)$
\begin{align}
\sum_{t=1}^{T} \dsnd_t(\x_t, \delta)
&=
\sum_{t=1}^{T}
\frac{\|\rwdv\|_1}{2} \beta_t(\delta)
\left\| Z_t^{-\frac{1}{2}} (I \otimes \x) \right\|_{2,2}^2
\le
\frac{\|\rwdv\|_1}{2} \beta_T(\delta)
\sum_{t=1}^{T}
\left\| Z_t^{-\frac{1}{2}} (I \otimes \x) \right\|_{2,2}^2
\notag\\
&\le
\frac{\|\rwdv\|_1}{2} \beta_T(\delta)
\sum_{t=1}^{T}
\lambda_{\max}((I \otimes \x)^\top Z_t^{-1} (I \otimes \x))
\notag\\
&\le
\frac{\|\rwdv\|_1}{2} \beta_T(\delta)
\sum_{t=1}^{T}
\tr((I \otimes \x)^\top Z_t^{-1} (I \otimes \x))
\notag\\
&=
\frac{\|\rwdv\|_1}{2} \beta_T(\delta)
\sum_{t=1}^{T}
\tr(Z_t^{-1} (I \otimes \x \x^\top))
\label{ineq:regret-dsnd-bound}
\end{align}
\par 
Plugging \eqref{ineq:regret-dfst-bound} and \eqref{ineq:regret-dsnd-bound} into the regret bound \eqref{ineq:reg-less-delta},
\begin{align*}
\reg(T)
\le &~
2 R \sqrt{\beta_{T-1}(\delta)} \sqrt{T}
\sqrt{ \sum_{t=1}^{T} \tr\left( Z_t^{-1} \nabla^2 \ell_t(W_{t+1}) \right)}
\\ &
+
R \sqrt{K}
\left( \frac{4 \eta D }{\sqrt{\lambda}} \sqrt{\beta_{T-1}(\delta)} + \beta_{T-1}(\delta) \right)
\sum_{t=1}^{T} \tr\left( Z_t^{-1} (I \otimes \x_t \x_t^\top) \right)
\end{align*}
By lemma \ref{lem:fd-self-norm-t} and \ref{lem:fd-self-norm-I-xxt-zt},
\begin{align*}
\reg(T)
\le &~
2 R \sqrt{\beta_{T-1}(\delta)} \sqrt{T}
\sqrt{\tail{T} \left(Kd \ln \tail{T}{\lambda/2} + m \ln (1 + 2\frac{D^2}{m \lambda}T)\right)}
\\ &
+
R \sqrt{K}
\left( \frac{4 \eta D }{\sqrt{\lambda}} \sqrt{\beta_{T-1}(\delta)} + \beta_{T-1}(\delta) \right)
\kappa \tail{T}{\lambda} \left(Kd \ln \tail{T}{\lambda/2} + m \ln (1 + 2\frac{D^2}{m \lambda}T)\right)
\end{align*}
Here
\begin{align*}
	\beta_T(\delta)
	&=
	\O\left(
		\ln K (\ln K + \ln T) \tail{t} (K d \ln \tail{t} + m \ln T)
	\right)
. \end{align*}
Therefore the regret bound becomes
\begin{align*}
&~
\reg(T)
\\ = &~
\O\left(
	\tail{T}{\lambda} (Kd \ln \tail{T} + m \ln T) \ln K \sqrt{\ln T} \sqrt{T}
	+
	\kappa \sqrt{K} \ln^2 K \ln T \tail{T}^2 (Kd \ln \tail{T}{\lambda} + m \ln T)^2
\right)
\\ = &~
\Ot\left(
	\tail{T}{\lambda} (K d \ln \tail{T}{\lambda} + m) \sqrt{T}
\right)
\end{align*}
\end{proof}

\paragraph{Proof for Proposition \ref{prop:regret-k-eq-1}}
\begin{proof}
In the
previous discussion,
The decision rule is \begin{equation} \label{eq:ofu-mlogb-optimistic-reward-rule-K-eq-1}
	(\wh_t, \x_t) = \argmax_{\w\in\W,\x\in\A_t} \w^\top \x_t
, \end{equation} and the regret is \[
	\reg(T)
	=
	\sum_{t=1}^{T} \sigma(\ws^\top \xs) - \sigma(\ws^\top \x_t)
, \] since $\rwdv = [0, 1]^\top$ and we ignore the $0$-th component of $\sig(\cdot)$.
And we perform the similar decomposition as in \citet[Corollary 1]{zhang2023mlb}:
\begin{align*}
&~
\reg(T)
=
\sum_{t=1}^{T} \sigma(\ws^\top \xs) - \sigma(\ws^\top \x_t)
\\ &~\text{(by optimistic decision rule \eqref{eq:ofu-mlogb-optimistic-reward-rule-K-eq-1})}
\\\le &~
\sum_{t=1}^{T} \sigma(\wh_t^\top \x_t) - \sigma(\ws^\top \x_t)
\\ &~\text{(by Taylor's theorem)}
\\= &~
\sum_{t=1}^{T}
\sigma'(\ws^\top \x_t)(\wh_t - \ws)^\top \x_t
+
\sigma''(\xi_t) ((\wh_t - \ws)^\top \x_t)^2
\\= &~
\sum_{t=1}^{T}
\sqrt{\sigma'(\ws^\top \x_t)}(\wh_t - \ws)^\top \sqrt{\sigma'(\ws^\top \x_t)}\x_t
+
\sigma''(\xi_t) ((\wh_t - \ws)^\top \x_t)^2
\\&~\text{(by Taylor's theorem)}
\\= &~
\sum_{t=1}^{T}
\sqrt{\sigma'(\ws^\top \x_t)} (\wh_t - \ws)^\top \sqrt{\sigma'(\w_{t+1}^\top \x_t)}\x_t
+
\sqrt{\sigma'(\ws^\top \x_t)} (\wh_t - \ws)^\top \x_t
\frac{\sigma''(\psi_t)}{2 \sqrt{\sigma'(\psi_t)}}
(\ws - \w_{t+1})^\top \x_t
\\ &
+
\sigma''(\xi_t) ((\wh_t - \ws)^\top \x_t)^2
\\&~\text{(by
	$\sigma'(\cdot) \le \frac{1}{4}$,
	$\sigma''(\cdot) \le \frac{1}{4}$, and
	$\sigma''(\cdot)/(2 \sqrt{\sigma'(\cdot)}) \le \frac{1}{4}$
)}
\\\le &~
\sum_{t=1}^{T}
\sqrt{\sigma'(\ws^\top \x_t)}(\wh_t - \ws)^\top \sqrt{\sigma'(\w_{t+1}^\top \x_t)}\x_t
+
\frac{1}{8}
(\wh_t - \ws)^\top \x_t
(\ws - \w_{t+1})^\top \x_t
+
\frac{1}{4} ((\wh_t - \ws)^\top \x_t)^2
\\&~\text{(by Cauchy-Schwarz inequality)}
\\\le &~
\sum_{t=1}^{T}
\left\| \wh_t - \ws \right\|_{Z_t}
\sqrt{\sigma'(\ws^\top \x_t)}
\left\| \sqrt{\sigma'(\w_{t+1}^\top \x_t)}\x_t \right\|_{Z_t^{-1}}
\\ &
+
\frac{1}{8}
\left\| \wh_t - \ws \right\|_{Z_t}
\left\| \ws - \w_{t+1} \right\|_{Z_{t+1}}
\left\| \x_t \right\|_{Z_t^{-1}}
\left\| \x_t \right\|_{Z_{t+1}^{-1}}
+
\frac{1}{4} \left\| \wh_t - \ws \right\|_{Z_t}^2
\left\| \x_t \right\|_{Z_t^{-1}}^2
\end{align*}
where the second equality is by
$\frac{ \sigma''(\cdot)}{2 \sqrt{\sigma'(\cdot)}}
= \frac{1}{2} (1 - 2 \sigma(\cdot)) \sqrt{\sigma(\cdot) (1 - \sigma(\cdot))}
\le \frac{1}{4}$ and $\sigma'(\cdot) \le \frac{1}{4}$.
\par 
The term $\|\w_t - \ws\|_{Z_t}$ is bounded by Theorem \ref{thm:confidence-region}, as well as $\|\wh_t - \ws\|_{Z_t}$
\begin{align*}
	\left\| \w_t - \ws \right\|_{Z_t}
	&\le \sqrt{\beta_t(\delta)} \le \sqrt{\beta_T(\delta)} \\
	\left\| \wh_t - \ws \right\|_{Z_t}
	&\le \left\| \wh_t - \w_t \right\|_{Z_t} + \left\| \w_t - \ws \right\|_{Z_t}
	\le 2 \sqrt{\beta_t(\delta)} \le 2 \sqrt{\beta_T(\delta)}
\end{align*}
Therefore we perform Cauchy-Schwarz inequality on the first and second terms
\begin{align*}
	\sum_{t=1}^{T}
		\left\| \wh_t - \ws \right\|_{Z_t}
		\sqrt{\sigma'(\ws^\top \x_t)}
		\left\| \sqrt{\sigma'(\w_{t+1}^\top \x_t)}\x_t \right\|_{Z_t^{-1}}
	&\le \sqrt{\beta_T(\delta)}
		\sum_{t=1}^{T} \sqrt{\sigma'(\ws^\top \x_t)} \left\| \sqrt{\sigma'(\w_{t+1}^\top \x_t)}\x_t \right\|_{Z_t^{-1}}
	\\
	&\le \sqrt{\beta_T(\delta)}
		\sqrt{\sum_{t=1}^{T} \sigma'(\ws^\top \x_t)}
		\sqrt{\sum_{t=1}^{T} \left\| \sqrt{\sigma'(\w_{t+1}^\top \x_t)}\x_t \right\|_{Z_t^{-1}}^2}
	\\
	\sum_{t=1}^{T}
		\frac{1}{8}
		\left\| \wh_t - \ws \right\|_{Z_t}
		\left\| \ws - \w_{t+1} \right\|_{Z_{t+1}}
		\left\| \x_t \right\|_{Z_t^{-1}}
		\left\| \x_t \right\|_{Z_{t+1}^{-1}}
	&\le \frac{1}{4} \beta_T(\delta)
		\sqrt{\sum_{t=1}^{T} \left\| \x_t \right\|_{Z_t^{-1}}^2}
		\sqrt{\sum_{t=1}^{T} \left\| \x_t \right\|_{Z_{t+1}^{-1}}^2}
\end{align*}
And for the third term, we have
\begin{align*}
	\sum_{t=1}^{T}
		\frac{1}{4} \left\| \wh_t - \ws \right\|_{Z_t}^2
		\left\| \x_t \right\|_{Z_t^{-1}}^2
	&\le \beta_T(\delta)
		\sum_{t=1}^{T} \left\| \x_t \right\|_{Z_t^{-1}}^2
\end{align*}
\par 
By Lemma \ref{lem:fd-self-norm-t}, \ref{lem:fd-self-norm-I-xxt-zt+1}, and \ref{lem:fd-self-norm-I-xxt-zt},
following terms have upper bounds
\begin{align*}
	\sum_{t=1}^{T} \left\| \sqrt{\sigma'(\w_{t+1}^\top \x_t)}\x_t \right\|_{Z_t^{-1}}^2
	&= \O\left(\tail{T}{\lambda} \left(d \ln \tail{T} + m \ln T\right)\right)
	\\
	\sum_{t=1}^{T} \left\| \x_t \right\|_{Z_{t+1}^{-1}}^2
	&= \O\left(\tail{T}{\lambda} \kappa \left(d \ln \tail{T} + m \ln T\right)\right)
	\\
	\sum_{t=1}^{T} \left\| \x_t \right\|_{Z_t^{-1}}^2
	&= \O\left(\tail{T}{\lambda} \kappa \left(d \ln \tail{T} + m \ln T\right)\right)
\end{align*}
And \citet[Theorem 2]{faury2022jointly-logistic-bandit}) indicates that
\begin{align*}
	\sum_{t=1}^{T} \sigma'(\ws^\top \x_t)
	&\le \reg(T) + T \sigma'(\ws^\top \xs)
\end{align*}
Plug the above terms into the regret bound, we have
\begin{align*}
\reg(T)
	= &~
	\O\left(\sqrt{\beta_T(\delta) \tail{T}{\lambda} \left(Kd \ln \tail{T}{\lambda} + m \ln T\right)}\right)
	\sqrt{\reg(T) + T \sigma'(\ws^\top \xs)}
	\\ &
	+ \O\left(\kappa \beta_T(\delta) \tail{T}{\lambda} \left(Kd \ln \tail{T}{\lambda} + m \ln T\right)\right)
\end{align*}
By resolving the inequality w.r.t. $\reg(T)$ above, we have \begin{align*}
	\reg(T)
	\le &~
	\O\left(\sqrt{\beta_T(\delta) \tail{T}{\lambda} \left(Kd \ln \tail{T}{\lambda} + m \ln T\right)}\right)
	\sqrt{T \sigma'(\ws^\top \xs)}
	+ \O\left(\kappa \beta_T(\delta) \tail{T}{\lambda} \left(Kd \ln \tail{T}{\lambda} + m \ln T\right)\right)
	\\ = &~
	\O\left(
		\tail{T} (Kd \ln \tail{T}{\lambda} + m \ln T)
		\ln K \sqrt{\ln T}
		\sqrt{T / \kappa_*}
		+
		\kappa \tail{T}^2
		(Kd \ln \tail{T}{\lambda} + m \ln T)^2
		\ln^2 K \ln T
	\right)
	\\ = &~
	\Ot\left(
		\tail{T} (Kd \ln \tail{T}{\lambda} + m) \sqrt{T / \kappa_*}
	\right)
, \end{align*} where $\kappa_*$ is $1/(\sigma'(\ws^\top \xs))$
\end{proof}

\subsection{Useful Lemmas}

\paragraph{Proof for Lemma \ref{lem:fd-oful-mlogb-efficient-reward}}
\begin{lemma*}
For any $\x \in \A_t$, let $Z_t = \lambda I_{Kd} + V_t \Lambda_t V_t^\top$.
The spectral norm satisfies
\[
	\left\| Z_t^{-\frac{1}{2}} (I_K \otimes \x) \right\|_2^2 = \|A(\x)\|_2
, \]
where $A(\x) \in \R^{K \times K}$ is a symmetric matrix with entries given by
\[
	A(\x)_{ij}
	=
	\sum_{k=1}^{m} \left(\frac{1}{\mu_k + \lambda} - \frac{1}{\lambda}\right) \left< \x, \bu_{ki} \right> \left< \x, \bu_{kj} \right>
	+
	\frac{\|\x\|_2^2}{\lambda} \delta_{ij}
. \]
Here, $\mu_k$ is the $k$-th diagonal element of $\Lambda_t$.
$\bu_{ki} \in \R^d$ denotes the sub-vector extracted from indices $(i-1)d+1$ to $id$ of $V_t$'s $k$-th column,
and $\delta_{ij}$ denotes the Kronecker delta.
\begin{proof}
We have \[
	\left\| Z_t^{-\frac{1}{2}} (I_K \otimes \x) \right\|_2^2
	=
	\lambda_{\max}( (I_K \otimes \x)^\top Z_t^{-\frac{1}{2}} Z_t^{-\frac{1}{2}} (I_K \otimes \x) )
	=
	\lambda_{\max} \left( (I_K \otimes \x)^\top Z_t^{-1} (I_K \otimes \x) \right)
. \]
The first equality is by $\|B\|_2^2 = \lambda_{\max}(B^\top B)$.
Let $A$ denote the matrix $(I_K \otimes \x)^\top Z_t^{-1} (I_K \otimes \x)$.
And let $\hat{\mu}_i$ and $\hat{\bu}_i$ be the complementary eigenvalues $\mu_i$ and eigenvectors $\bu_i$,
such that $\hat{\mu}_i = 0$ and $\hat{\bu}_i$ is orthogonal to the column space of $V_t$ for $i = m+1, \ldots, Kd$,
and $\hat{\mu}_i = \mu_i$ and $\hat{\bu}_i = \bu_i$ for $i = 1, \ldots, m$.
Then we have the eigen-decomposition of $Z_t$ as \[
	Z_t^{-1} = \sum_{i=1}^{Kd} \frac{1}{\hat{\mu}_i + \lambda} \hat{\bu}_i \hat{\bu}_i^\top
. \]
And the element $(A_{ij})_{i,j=1}^K$ can be computed as follows:
\begin{align*}
A_{ij}
&=
(\e_i \otimes \x)^\top Z_t^{-1} (\e_j \otimes \x)
\\
&=
(\e_i \otimes \x)^\top \sum_{k=1}^{Kd} \frac{1}{\hat{\mu}_k + \lambda} \hat{\bu}_k \hat{\bu}_k^\top (\e_j \otimes \x)
\\
&=
(\e_i \otimes \x)^\top \sum_{k=1}^m \frac{1}{\mu_k + \lambda} \bu_k \bu_k^\top (\e_j \otimes \x)
+
\frac{1}{\lambda} (\e_i \otimes \x)^\top \sum_{k=m+1}^{Kd} \hat{\bu}_k \hat{\bu}_k^\top (\e_j \otimes \x)
\\
&=
\sum_{k=1}^m \frac{1}{\mu_k + \lambda} (\e_i \otimes \x)^\top \bu_k \bu_k^\top (\e_j \otimes \x)
+
\frac{1}{\lambda} (\e_i \otimes \x)^\top \left(I_{Kd} - \sum_{k=1}^m \bu_k \bu_k^\top\right) (\e_j \otimes \x)
\\
&=
\frac{1}{\lambda} (\e_i \otimes \x)^\top (\e_j \otimes \x)
+
\sum_{k=1}^m \left(\frac{1}{\mu_k + \lambda} - \frac{1}{\lambda}\right) (\e_i \otimes \x)^\top \bu_k \bu_k^\top (\e_j \otimes \x)
\\
&=
\frac{\x^\top \x}{\lambda} \delta_i^j
+
\sum_{k=1}^m \left(\frac{1}{\mu_k + \lambda} - \frac{1}{\lambda}\right) \x^\top \bu_{ki} \x^\top \bu_{kj}
\end{align*}
where $\delta_i^j$ is the Kronecker delta function.
The fourth equality is by the fact that $\sum_{k=m+1}^{Kd} \hat{\bu}_k \hat{\bu}_k^\top = I_{Kd} - \sum_{k=1}^m \hat{\bu}_k \hat{\bu}_k^\top = I_{Kd} - \sum_{k=1}^m \bu_k \bu_k^\top$.
\par 
Then $\{ \x^\top \bu_{kj} \}_{k,j}$ can be computed in $\O(K d m)$ time,
$A$ can be constructed in $\O(K^2 m)$,
and $\lambda_{\max}(A) = \|A\|_2$ can be computed in $\O(K^3)$.
Therefore, the operator norm $\left\| Z_t^{-\frac{1}{2}}(I_K\otimes\x) \right\|_{2,2}$ can be computed in $\O(K^3 + K^2 m + K d m)$ time.
\end{proof}
\end{lemma*}

\paragraph{Proof for Lemma \ref{lem:fd-oful-mlogb-efficient-W-update}}
\begin{lemma*}
Let $\Zt_t = \lambda I_{Kd} + \Vt_t \Lamt_t \Vt_t^\top$, where $\bvt_i$ is the $i$-th column of $\Vt_t$ and $\mu_i$ is the $i$-th diagonal
element of $\Lamt_t$.
If $\|\Wv'\|_2 \le S$, then $\Wv_{t+1} = \Wv'$.
Otherwise, the projection is given by:
\[
	\Wv_{t+1} = \left(\frac{\lambda}{\lambda + \nu} I + \frac{\nu}{\lambda + \nu} \Vt_t \Lamt_t M \Vt_t^\top\right) \Wv'
, \]
where $M = ((\lambda + \nu)I_{m+K} + \Lamt_t)^{-1}$ is diagonal.
The dual variable $\nu > 0$ is the unique root of
\[
	f(\nu)= S^2
, \]
where $f(\nu)$ is a monotonically decreasing function on the interval $(0, +\infty)$
\[
	f(\nu) = \frac{(\|\Wv'\|_2^2-\|\Vt_t^\top\Wv'\|_2^2) \lambda^2}{(\lambda + \nu)^2} + \sum_{i=1}^{m+K} \frac{(\bvt_i^\top \Wv')^2 (\mu_i + \lambda)^2}{(\mu_i + \lambda + \nu)^2}
. \]
\begin{proof}
In the following proof, we denote $V = \Vt_t, \Lambda = \Lamt_t$ for readability.
And with loss of generality, we assume $\Lambda$ to be a full-rank diagonal matrix, otherwise we can remove the zero eigenvalues and the corresponding eigenvectors without affecting the solution.
And we assume the rank of $\Lambda$ is $m'$.
We rewrite the projection problem as a convex optimization problem
\begin{align*}
	\min_{\Wv} \quad& \frac{1}{2} \left\| \Wv - \Wv' \right\|_{\Zt_t}^2 \\
	\mathrm{s.t.} \quad& \frac{1}{2} \left\| \Wv \right\|_2^2 \le \frac{S^2}{2}
\end{align*}
If $\|\Wv'\|_2 \le S$, then $\Wv'$ is the optimal solution with optimal value $0$.
\par 
Assume $\|\Wv'\|_2 > S$.
Since there exists an interior point $\Wv = 0$ that strictly satisfies the constraint $\frac{1}{2} \| 0 \|_2^2 < S^2 / 2$,
the Slater's condition holds and the strong duality holds.
Therefore the optimal solution is equivalent to the Karush-Kuhn-Tucker (KKT) conditions.
\begin{align}
	\Zt_t(\Wv - \Wv') + \nu \Wv &= 0 \label{eq:fomWupdate-kkt-differential-equation}\\
	\left\| \Wv \right\|_2 &\le S \label{ineq:fomWupdate-kkt-w-region} \\
	\nu &\ge 0 \label{ineq:fomWupdate-kkt-nv-region} \\
	\nu \left( \left\| \Wv \right\|_2 - S \right) &= 0 \label{eq:fomWupdate-kkt-complementary-slackness}
\end{align}
Since $\Zt_t$ is positive definite and $\nu \ge 0$, the matrix $\Zt_t + \nu I$ is invertible.
By equation \eqref{eq:fomWupdate-kkt-differential-equation}, we have \begin{align*}
	\Wv &= (\Zt_t + \nu I)^{-1} \Zt_t \Wv' \\
		&= ((\lambda + \nu) I + V \Lambda V^\top)^{-1} (\lambda I + V \Lambda V^\top) \Wv'
\end{align*}
By Woodbury matrix identity, we have
\begin{align*}
	((\lambda + \nu) I + V \Lambda V^\top)^{-1}
	&= \frac{1}{\lambda + \nu} I - \frac{1}{(\lambda + \nu)^2} V (\Lambda^{-1} + \frac{1}{\lambda + \nu} V^\top V)^{-1} V^\top \\
	&= \frac{1}{\lambda + \nu} I - \frac{1}{\lambda + \nu} V (\Lambda + (\lambda + \nu) I)^{-1} \Lambda V^\top
\end{align*}
Therefore,
\begin{align*}
	\Wv
	&= \left(\frac{1}{\lambda + \nu} I - \frac{1}{\lambda + \nu} V (\Lambda + (\lambda + \nu) I)^{-1} \Lambda V^\top\right)
		(\lambda I + V \Lambda V^\top) \Wv' \\
	&= \left(\frac{\lambda}{\lambda + \nu} I - \frac{\lambda}{\lambda + \nu} V (\Lambda + (\lambda + \nu) I)^{-1} \Lambda V^\top
		+ \frac{1}{\lambda + \nu} V \Lambda V^\top - \frac{1}{\lambda + \nu} V ((\lambda + \nu) I + \Lambda)^{-1} \Lambda^2 V^\top\right) \Wv' \\
	&= \left(\frac{\lambda}{\lambda + \nu} I + \frac{\nu}{\lambda + \nu} V ((\lambda + \nu) I + \Lambda)^{-1} \Lambda V^\top\right) \Wv'
\end{align*}
Let $M = ((\lambda + \nu) I + \Lambda)^{-1}$, therefore \[
	\Wv
	=
	\left(\frac{\lambda}{\lambda + \nu} I + \frac{\nu}{\lambda + \nu} V M \Lambda V^\top\right) \Wv'
. \]
If $\nu = 0$, then $\Wv = \Wv'$.
However by \eqref{ineq:fomWupdate-kkt-w-region}, $\|\Wv\|_2 \le S$, which contradicts the assumption $\|\Wv'\|_2 > S$.
Therefore by \eqref{ineq:fomWupdate-kkt-nv-region} $\nu > 0$.
And by \eqref{eq:fomWupdate-kkt-complementary-slackness}, we have $\|\Wv\|_2^2 = S^2$.
Let $\hat{V}, \hat{\Lambda}$ be the full matrix of $V, \Lambda$ by adding the orthogonal basis of the null space of $V$ and $0$ eigenvalues.
Then we have
\begin{align*}
	\left\| \Wv \right\|_2^2
	&= \left\| (\Zt_t + \nu I)^{-1} \Zt_t \Wv' \right\|_2^2 \\
	&= \left\| \hat{V} ((\lambda + \nu) I + \hat{\Lambda})^{-1} \hat{V}^\top \hat{V} (\lambda I + \hat{\Lambda}) \hat{V}^\top \Wv' \right\|_2^2 \\
	&= \left\| \hat{V} ((\lambda + \nu) I + \hat{\Lambda})^{-1} (\lambda I + \hat{\Lambda}) \hat{V}^\top \Wv' \right\|_2^2 \\
	&= \left\| ((\lambda + \nu) I + \hat{\Lambda})^{-1} (\lambda I + \hat{\Lambda}) \hat{V}^\top \Wv' \right\|_2^2 \\
	&= \sum_{i=1}^{Kd} \frac{(\hat{\bv}_i^\top \Wv)^2 (\lambda + \hat{\mu}_i)^2}{(\lambda + \hat{\mu}_i + \nu)^2}
\end{align*}
Here $\hat{\bv}_i$ is the $i$-th column of $\hat{V}$ and $\hat{\mu}_i$ is the $i$-th diagonal element of $\hat{\Lambda}$.
Since $\hat{\lambda}_i = 0$ for $i > m'$, we have
\begin{align*}
	\left\| \Wv \right\|_2^2
	&= \sum_{i=1}^{Kd} \frac{(\hat{\bv}_i^\top \Wv')^2 (\lambda + \hat{mu}_i)^2}{(\lambda + \hat{mu}_i + \nu)^2} \\
	&= \sum_{i=1}^{m'} \frac{(\bv_i^\top \Wv')^2 (\lambda + \mu_i)^2}{(\lambda + \mu_i + \nu)^2}
		+ \sum_{i=m'+1}^{Kd} \frac{(\hat{\bv}_i^\top \Wv')^2 \lambda^2}{(\lambda + \nu)^2}
\end{align*}
Here by assumption $m' \le m+K$ we set $\mu_i = 0$ for $m'<i\le m+K$.
And we have
\begin{align*}
	\left\| \Wv \right\|_2^2
	&= \sum_{i=1}^{m+K} \frac{(\bv_i^\top \Wv')^2 (\lambda + \mu_i)^2}{(\lambda + \mu_i + \nu)^2}
		+ \sum_{i=m+K+1}^{Kd} \frac{(\hat{\bv}_i^\top \Wv')^2 \lambda^2}{(\lambda + \nu)^2} \\
	&= \sum_{i=1}^{m+K} \frac{(\bv_i^\top \Wv')^2 (\lambda + \mu_i)^2}{(\lambda + \mu_i + \nu)^2}
		+ \sum_{i=m+K+1}^{Kd} \frac{\Wv'^\top \hat{\bv}_i \hat{\bv}_i^\top \Wv' \lambda^2}{(\lambda + \nu)^2} \\
	&= \sum_{i=1}^{m+K} \frac{(\bv_i^\top \Wv')^2 (\lambda + \mu_i)^2}{(\lambda + \mu_i + \nu)^2}
		+ \Wv'^\top \left(I - \sum_{i=1}^{m+K} \hat{\bv}_i \hat{\bv}_i^\top\right) \frac{\Wv' \lambda^2}{(\lambda + \nu)^2} \\
	&= \sum_{i=1}^{m+K} \frac{(\bv_i^\top \Wv')^2 (\lambda + \mu_i)^2}{(\lambda + \mu_i + \nu)^2}
		+ \frac{\Wv'^\top (I - V V^\top) \Wv' \lambda^2}{(\lambda + \nu)^2}
\end{align*}
Let \[
	f(\nu)
	= \sum_{i=1}^{m+K} \frac{(\bv_i^\top \Wv')^2 (\lambda + \lambda_i)^2}{(\lambda + \lambda_i + \nu)^2}
		+ \frac{\Wv'^\top (I - V V^\top) \Wv' \lambda^2}{(\lambda + \nu)^2}
	, \nu > 0
. \] Therefore \[
	f(\nu) = \left\| \Wv \right\|_2^2 = S^2
. \] This means that the dual optimiser $\nu$ is the root of the equation $f(\nu) = S^2$.
Since $f(\nu)$ is a monotonically decreasing function with respect to $\nu$ in the interval $(0, +\infty)$,
and $\lim_{\nu \to 0^+} f(\nu) = \|\Wv'\|_2^2 > S^2$, $\lim_{\nu \to +\infty} f(\nu) = 0 < S^2$,
there exists a unique root of the equation $f(\nu) = S^2$ in the interval $(0, +\infty)$.
Therefore, given $\Zt_t = \lambda I + V \Lambda V^\top$, the problem $f(\nu) = S^2$ can be solved in $\O(Kd (m + K) + (m + K) \log\log \varepsilon)$ using Newton's method,
and the optimal solution $\Wv_{t+1}$ can be computed in $\O(Kd (m + K))$ time,
where $\varepsilon$ is the machine precision.
Therefore the total time complexity is $\O(Kd (m + K))$.
\end{proof}
\end{lemma*}

\begin{lemma}[Lemma 13, 14 of \citet{zhang2023mlb}] \label{lem:cr-ell-regret}
For any $\delta \in (0, 1)$, with probability at least $1 - \delta$, we have
\begin{align*}
	\forall t > 0.~
	\sum_{s=1}^{t} \ell_s(\Ws) - \ell_s(W_{s+1})
	\le &~
	\left(D S + \ln(K+1) + 2 \ln ((K+1) t + 2)\right) \left(\frac{5}{4}  + 4 \ln \frac{\sqrt{1 + 2 t}}{\delta}\right)
	+
	2 + \ln (K+1)
	\\ &
	+
	\frac{\sqrt6}{24 Kd} \lambda \tail{t}{\lambda} \left( Kd \ln \tail{t}{\lambda/2} + m \ln (1 + 2\frac{D^2}{m \lambda}t) \right)
	+ \frac{12 Kd}{\lambda} \sum_{s=1}^{t} \left\| \Wv_{s+1} - \Wv_s \right\|_{Z_s}^2
. \end{align*}
\begin{proof}
Let $\sig^+$ be the inverse of $\sig$, \[
	\sig^+(\p) = \left[\ln \frac{p_1}{p_0}, \ln \frac{p_2}{p_0}, \dots, \ln \frac{p_K}{p_0}\right]^\top,
	\left\| \p \right\|_1 = 1, \p \in [0, 1]^{K+1}
. \] And define a intermedia function w.r.t. $\ell_s$ \[
	\ellh_s(z) := - \y_t^\top \ln\sig(z), \ellh_s(W \x_s) = \ell_s(W)
. \] Let $z_s$ be an intermediary prediction w.r.t. the distribution $\cP_s = \mathcal{N}(W_s, (c Z_s)^{-1})$ \[
	z_s = \sig^+(\E_{W \sim \cP_s}[\sig(W \x_s)])
, \] and $z_s^\mu$ be an smooth approximation of $z_s$ \[
	z_s^\mu = \sig^+(\smooth_\mu(\sig(z_s)))
, \] where $\smooth_\mu(\p) = (1 - \mu)\p + \mu / (K+1)$ for any probability distribution $\left\| \p \right\|_1 = 1, \p \in [0, 1]^{K+1}$.
Thus $z_s$ and $z_s^\mu$ are $\F_s$-measurable, \[
	\left\| z_s^\mu \right\|_\infty
	=
	\max_{k\in\left<K\right>} \left| \ln \frac{\sig_k(z_s^\mu)}{\sig_0(z_s^\mu)} \right|
	\le
	\Max{\left| \ln \frac{1 - \mu K / (K+1)}{\mu / (K+1)} \right|, \left| \ln \frac{\mu / (K+1)}{1 - \mu K / (K+1)} \right|}
	\le
	\ln \left(\frac{K+1}{\mu} - K\right)
. \]
\par 
We follow the proof of \citet[Theorem 3]{zhang2023mlb}.
By dividing $\sum_{s=1}^{t} \ell_s(\Ws) - \ell_s(W_s)$ into three parts,\[
	\sum_{s=1}^{t} \ell_s(\Ws) - \ell_s(W_s)
	=
	\sum_{s=1}^{t} \ell_s(\Ws) - \ellh_s(z_s^\mu)
	+
	\sum_{s=1}^{t} \ellh_s(z_s^\mu) - \ellh_s(z_s)
	+
	\sum_{s=1}^{t} \ellh_s(z_s) - \ell_s(W_s)
. \] For the first term $\sum_{s=1}^{t} \ell_s(\Ws) - \ellh_s(z_s^\mu)$,
by \citet[Lemma 13]{zhang2023mlb}, we have \begin{equation} \label{ineq:cr-ell-regret-first}
	\sum_{s=1}^{t} \ell_s(\Ws) - \ellh_s(z_s^\mu)
	\le
	\frac{5}{4} (D S + S_\mu) + 4 (D S + S_\mu) \ln \frac{\sqrt{1 + 2 t}}{\delta}
,  \end{equation} where $S_\mu = \ln (K+1) + 2 \ln (\frac{K+1}{\mu} + 2) \ge \ln (K+1) + 2 (\left\| z_s^\mu \right\|_\infty + 1)$.
And for the second term $\sum_{s=1}^{t} \ellh_s(z_s^\mu) - \ellh_s(z_s)$, by \citet[Lemma 17]{zhang2023mlb}, if $\mu \in [0, \frac{1}{2}]$ \begin{align}
	\sum_{s=1}^{t} \ellh_s(z_s^\mu) - \ellh_s(z_s)
	&=
	\sum_{s=1}^{t} \ln \frac{\sigma_{y_s}(z_s)}{(1 - \mu)\sigma_{y_s}(z_s) + \mu 1 / (K+1)}
	\notag \\
	&\le
	\sum_{s=1}^{t} \ln \frac{1}{1 - \mu K / (K+1)}
	\le
	\sum_{s=1}^{t} \ln \exp(2 \mu K / (K+1))
	\notag \\
	&\le
	\sum_{s=1}^{t} 2 \mu
	=
	2 \mu t
	\label{ineq:cr-ell-regret-second}
. \end{align}
The second inequality is due to $\forall x \in [0, \frac{1}{2}].~ \frac{1}{1 - x} \le \exp(2x)$.
And by lemma \ref{lem:cr-ell-regret-3}, the third term is bounded \begin{equation} \label{ineq:cr-ell-regret-third}
	\sum_{s=1}^{t} \ellh_s(z_s) - \ell_s(W_{s+1})
	\le
	\frac{\sqrt6}{24 Kd} \lambda \tail{t}{\lambda} \left( Kd \ln 2 \tail{t} + m \ln (1 + \frac{D^2}{m \lambda}t) \right)
	+ \frac{12 Kd}{\lambda} \sum_{s=1}^{t} \left\| \Wv_s - \Wv_{s+1} \right\|_{Z_s}^2
. \end{equation}
Combining \eqref{ineq:cr-ell-regret-first}, \eqref{ineq:cr-ell-regret-second} and \eqref{ineq:cr-ell-regret-third},
we get
\begin{align*}
	\sum_{s=1}^{t} \ell_s(\Ws) - \ell_s(W_{s+1})
	\le &~
	\frac{5}{4} (D S + S_\mu) + 4 (D S + S_\mu) \ln \frac{\sqrt{1 + 2 t}}{\delta}
	+
	2 \mu t
	\\ &
	+
	\frac{\sqrt6}{24 Kd} \lambda \tail{t}{\lambda} \left( Kd \ln 2 \tail{t} + m \ln (1 + \frac{D^2}{m \lambda}t) \right)
	+ \frac{12 Kd}{\lambda} \sum_{s=1}^{t} \left\| \Wv_{s+1} - \Wv_s \right\|_{Z_s}^2
\end{align*}
And set $\mu = \frac{1}{t}$ to get
\begin{align*}
	\sum_{s=1}^{t} \ell_s(\Ws) - \ell_s(W_{s+1})
	\le &~
	\left(D S + \ln(K+1) + 2 \ln ((K+1) t + 2)\right) \left(\frac{5}{4}  + 4 \ln \frac{\sqrt{1 + 2 t}}{\delta}\right)
	+
	2 + \ln (K+1)
	\\ &
	+
	\frac{\sqrt6}{24 Kd} \lambda \tail{t}{\lambda} \left( Kd \ln 2 \tail{t} + m \ln (1 + \frac{D^2}{m \lambda}t) \right)
	+ \frac{12 Kd}{\lambda} \sum_{s=1}^{t} \left\| \Wv_{s+1} - \Wv_s \right\|_{Z_s}^2
\end{align*}
\end{proof}
\end{lemma}

\begin{lemma} \label{lem:cr-ell-regret-3}
For any $t > 0$ \[
	\sum_{s=1}^{t} \ellh_s(z_s) - \ell_s(W_{s+1})
	\le
	\frac{\sqrt6}{24 Kd} \lambda \tail{t}{\lambda} \left( Kd \ln 2 \tail{t}{\lambda/2} + m \ln (1 + 2\frac{D^2}{m \lambda}t) \right)
	+
	\frac{12 Kd}{\lambda} \sum_{s=1}^{t} \left\| \Wv_s - \Wv_{s+1} \right\|_{Z_s}^2
. \]
\begin{proof}
By definition of $z_s$, we know \[
	\exp(- \ellh_s(z_s))
	=
	\Exp{\exp(- \ell_s(W))}{W \sim \cP_s}
. \] And therefore
\begin{align*}
\sum_{s=1}^{t} \ellh_s(z_s) - \ell_s(W_{s+1})
&=
\sum_{s=1}^{t} - \ln \exp(- \ellh_s(z_s)) - \ell_s(W_{s+1})
\\
&=
\sum_{s=1}^{t} - \ln \Exp{\exp (- \ell_s(W))}{W \sim \cP_s} - \ell_s(W_{s+1})
\\
&=
\sum_{s=1}^{t} - \ln \Exp{\exp (- (\ell_s(W) - \ell_s(W_{s+1})))}{W \sim \cP_s}
\end{align*}
We shift the mean of $\cP_s$ from $W_s$ to $W_{s+1}$,
and get $\Q_s := \mathcal{N}(W_{s+1}, (c Z_s)^{-1})$.
For probability density function $f_p(W), f_q(W)$, the following relation holds \[
	\frac{f_p(W)}{f_q(W)} = \exp\left(- \left<c Z_s (\Wv_{s+1} - \Wv), \Wv_s - \Wv_{s+1}\right>\right) \exp\left(- \frac{1}{2} \left\| \Wv_s - \Wv_{s+1} \right\|_{c Z_s}^2\right)
. \] Thus
\begin{align}
&~
\sum_{s=1}^{t} - \ln \Exp{\exp (- (\ell_s(W) - \ell_s(W_{s+1})))}{W \sim \cP_s}
\notag \\ = &~
\sum_{s=1}^{t}- \ln \Exp{\frac{f_p(W)}{f_q(W)} \exp\left(- \left(
	\ell_s(W) - \ell_s(W_{s+1})
\right)\right)}{W \sim \Q_s}
\notag \\ = &~
\sum_{s=1}^{t}- \ln \Exp{\exp\left(- \left(
	\ell_s(W) - \ell_s(W_{s+1})
	+
	\left< \Wv_{s+1} - \Wv, c Z_s^\top (\Wv_s - \Wv_{s+1}) \right>
\right)\right)}{W \sim \Q_s}
\notag \\ &~
+
\frac{1}{2} \left\| \Wv_{s+1} - \Wv_s \right\|_{c Z_s}^2
\label{eq:lem-cr-ell-regret-3-sum}
\end{align}
By \citet[Theorem 4(d)]{tran-dinh2015composite-self-concordant}, and $\forall x > 0.~ (\exp(x) - x - 1) / x^2 \le \exp(x^2)$, \[
	\ell_s(W) - \ell_s(W_{s+1})
	\le
	\left< \nabla \ell_s(\Wv_{s+1}), \Wv - \Wv_{s+1} \right>
	+
	\rme^{6 \left\| W - W_{s+1} \right\|_F^2} \left\| \Wv - \Wv_{s+1} \right\|_{\nabla^2 \ell_s(W_{s+1})}^2
, \] the logarithm expectation term in \eqref{eq:lem-cr-ell-regret-3-sum} can be bounded as follows
\begin{align*}
&~
- \ln \Exp{\exp\left(- \left(
	\ell_s(W) - \ell_s(W_{s+1})
	+
	\left< c Z_s (\Wv_{s+1} - \Wv_s), \Wv - \Wv_{s+1} \right>
\right)\right)}{W \sim \Q_s}
\\ \le &~
- \ln \Exp{\exp\left(- \left(
	\left< \nabla \ell_s(\Wv_s) + cZ_s(\Wv_{s+1}-\Wv_s), \Wv - \Wv_{s+1} \right>
	+
	\rme^{6 \left\| W - W_{s+1} \right\|_F^2} \left\| \Wv - \Wv_{s+1} \right\|_{\nabla^2 \ell_s(W_{s+1})}^2
\right)\right)}
\\ &~\text{(by Jensen's inequality)}
\\ = &~
- \ln \exp\left(- \Exp{
	\left< \nabla \ell_s(\Wv_s) + cZ_s(\Wv_{s+1}-\Wv_s), \Wv - \Wv_{s+1} \right>
	+
	\rme^{6 \left\| W - W_{s+1} \right\|_F^2} \left\| \Wv - \Wv_{s+1} \right\|_{\nabla^2 \ell_s(W_{s+1})}^2
}\right)
\\ = &~
\Exp{
	\left< \nabla \ell_s(\Wv_s) + cZ_s(\Wv_{s+1}-\Wv_s), \Wv - \Wv_{s+1} \right>
	+
	\rme^{6 \left\| W - W_{s+1} \right\|_F^2} \left\| \Wv - \Wv_{s+1} \right\|_{\nabla^2 \ell_s(W_{s+1})}^2
}{W \sim \Q_s}
\\ = &~
\Exp{
	\rme^{6 \left\| W - W_{s+1} \right\|_F^2} \left\| \Wv - \Wv_{s+1} \right\|_{\nabla^2 \ell_s(W_{s+1})}^2
}{W \sim \Q_s}
\\ &~\text{(by Cachy-Schwarz inequality)}
\\ \le &~
\sqrt{\Exp{\rme^{12 \left\| W - W_{s+1} \right\|_F^2}}{W \sim \Q_s}}
\sqrt{\Exp{\left\| \Wv - \Wv_{s+1} \right\|_{\nabla^2 \ell_s(W_{s+1})}^4}{W \sim \Q_s}}
\end{align*}
For the first term $\sqrt{\Exp{\rme^{12 \left\| W - W_{s+1} \right\|_F^2}}{W \sim \Q_s}}$,
let $U \Sigma U^\top = Z_s$ temporarily be the eigendecomposition of $Z_s$.
Thus random variable $X := \sqrt{c \Sigma} U^\top (\Wv - \Wv_{s+1})$ follows the distribution $\mathcal{N}(0, I)$,
and
\begin{align*}
\sqrt{\Exp{\rme^{12 \left\| W - W_{s+1} \right\|_F^2}}{W \sim \Q_s}}
&=
\sqrt{\Exp{\exp\left( 12 \left\| U (c \Sigma)^{-\frac{1}{2}} X \right\|_2^2 \right)}{X}}
\\ &=
\sqrt{\Exp{\exp\left( 12 \sum_{i=1}^{Kd} \frac{1}{c \lambda_i} X_i^2 \right)}{X}}
\\ &~\text{(by $Z_s \succeq \lambda I$)}
\\ &=
\Exp{\exp\left( \frac{12}{c \lambda} X_i^2 \right)}{X}^{Kd / 2}
\\ &~\text{(by Jensen's inequality)}
\\ &\le
\Exp{\exp\left( \frac{6 Kd}{c \lambda} X_i^2 \right)}{X}
\end{align*}
By setting $\frac{6 Kd}{c \lambda} \le \frac{1}{4}$,
and $\Exp{\exp(X / 4)}{X \sim \mathcal{X}^2} \le \sqrt2$,
we have \begin{equation} \label{ineq:lem-cr-ell-regret-3-e1}
	\sqrt{\Exp{\rme^{12 \left\| W - W_{s+1} \right\|_F^2}}{W \sim \Q_s}}
	\le
	\sqrt2
. \end{equation}
And for the second term $\sqrt{\Exp{\left\| \Wv - \Wv_{s+1} \right\|_{\nabla^2 \ell_s(W_{s+1})}^4}{W \sim \Q_s}}$,
Let the full eigendecomposition $U \Sigma U^\top = Z_s^{-\frac{1}{2}} \nabla^2 \ell_s(W_{s+1}) Z_s^{-\frac{1}{2}}$ with rank at most $K$.
And random variable $X := U^\top (cZ_s)^{\frac{1}{2}} (\Wv - \Wv_{s+1})$
follows the distribution $\mathcal{N}(0, I)$,
and
\begin{align*}
\sqrt{\Exp{\left\| \Wv - \Wv_{s+1} \right\|_{\nabla^2 \ell_s(W_{s+1})}^4}{W \sim \Q_s}}
&=
\sqrt{\Exp{\left\| \nabla^2 \ell_s(W_{s+1})^{\frac{1}{2}} (cZ_s)^{-\frac{1}{2}} U X \right\|_2^4}{X}}
\\ &=
\sqrt{\Exp{\left( X^\top U^\top Z_s^{-\frac{1}{2}} \nabla^2 \ell_s(W_{s+1}) Z_s^{-\frac{1}{2}} U X / c \right)^2}{X}}
\\ &=
\sqrt{\Exp{\left( \sum_{i=1}^{K} \frac{\lambda_i}{c} X_i^2 \right)^2}{X}}
\\ &=
\sqrt{\Exp{\sum_{i, j} \frac{\lambda_i \lambda_j}{c^2} X_i^2 X_j^2}{X}}
\\ &=
\frac{1}{c} \sqrt{\sum_{i, j} \lambda_i \lambda_j \Exp{X_i^2 X_j^2}{X}}
\end{align*}
By $\forall i, j.~ \Exp{X_i^2 X_j^2}{X_i, X_j \sim \mathcal{N}(0, 1)} \le 3$
\begin{align}
&~
\sqrt{\Exp{\left\| \Wv - \Wv_{s+1} \right\|_{\nabla^2 \ell_s(W_{s+1})}^4}{W \sim \Q_s}}
\notag \\ \le &~
\frac{1}{c} \sqrt{\sum_{i, j} \lambda_i \lambda_j 3}
=
\frac{\sqrt3}{c} \sum_{i=1}^{Kd} \lambda_i
\notag \\ = &~
\frac{\sqrt3}{c} \tr(Z_s^{-\frac{1}{2}} \nabla^2 \ell_s(W_{s+1}) Z_s^{-\frac{1}{2}})
\notag \\ = &~
\frac{\sqrt3}{c} \tr(\nabla^2 \ell_s(W_{s+1}) Z_s^{-1})
\label{ineq:lem-cr-ell-regret-3-e2}
\end{align}
Combining \eqref{ineq:lem-cr-ell-regret-3-e1} and \eqref{ineq:lem-cr-ell-regret-3-e2},
the equation \eqref{eq:lem-cr-ell-regret-3-sum} can be bounded as follows \[
	\sum_{s=1}^{t} \ellh_s(z_s) - \ell_s(W_{s+1})
	\le
	\frac{\sqrt6}{c} \sum_{s=1}^{t} \tr(\nabla^2 \ell_s(W_{s+1}) Z_s^{-1})
	+
	\frac{1}{2} \sum_{s=1}^{t} \left\| \Wv_s - \Wv_{s+1} \right\|_{c Z_s}^2
. \]
To conclude, we recall that $\frac{6Kd}{c \lambda} \le \frac{1}{4}$.
Therefore, we set $c := \frac{24 Kd}{\lambda}$ \[
	\sum_{s=1}^{t} \ellh_s(z_s) - \ell_s(W_{s+1})
	\le
	\frac{\sqrt6}{24 Kd} \lambda \sum_{s=1}^t \tr( \nabla^2 \ell_s(W_{s+1}) Z_s^{-1} )
	+
	\frac{12 Kd}{\lambda} \sum_{s=1}^{t} \left\| \Wv_s - \Wv_{s+1} \right\|_{Z_s}^2
. \]
And by Lemma \ref{lem:fd-self-norm-t}, if $\lambda \ge 2 L D^2$\[
	\sum_{s=1}^t \tr( \nabla^2 \ell_s(W_{s+1}) Z_s^{-1} )
	\le
	\tail{t}{\lambda} (Kd \ln \tail{t}{\lambda/2} + m \ln (1 + 2\frac{D^2}{m \lambda}t))
. \]
Therefore, if $\lambda \ge 2 L D^2$, we have \[
	\sum_{s=1}^{t} \ellh_s(z_s) - \ell_s(W_{s+1})
	\le
	\frac{\sqrt6}{24 Kd} \lambda \tail{t}{\lambda} \left( Kd \ln 2 \tail{t} + m \ln (1 + \frac{D^2}{m \lambda}t) \right)
	+
	\frac{12 Kd}{\lambda} \sum_{s=1}^{t} \left\| \Wv_s - \Wv_{s+1} \right\|_{Z_s}^2
\]
\end{proof}
\end{lemma}

\begin{lemma} \label{lem:omd-wt-wt-1-bound}
For any $t > 0$, we have \[
	\left\| \Wv_{t+1} - \Wv_t \right\|_{Z_t}^2
	\le
	\frac{2 \eta D}{\sqrt{\lambda}}
. \]
\begin{proof}
By the definition of $W_t$, \[
	W_{t+1} := \argmin_{W \in \W} \left< \eta \nabla \ell_t(\Wv_t), \Wv \right> + \frac{1}{2} \left\| \Wv - \Wv_t \right\|_{Z_t + \eta \nabla^2 \ell_t(W_t)}^2
, \] $W_{t+1}$ satisfies the first order optimality condition \[
	\left<
		\eta \nabla \ell_t(\Wv_t)
		+ (Z_t + \eta \nabla^2 \ell_t(W_t)) (\Wv_{t+1} - \Wv_t)
		,
		\Wv_{t+1} - \Wv_t
	\right>
	\le
	0
. \] Rearrange to get \[
	\left\| \Wv_{t+1} - \Wv_t \right\|_{Z_t + \eta \nabla^2 \ell_t(W_t)}^2
	\le
	\left< \eta \nabla \ell_t(\Wv_t), \Wv_{t+1} - \Wv_t \right>
. \] By $\eta \nabla^2 \ell_t(W_t) \succeq 0$ and Cauchy-Schwarz inequality,
\begin{align*}
\left\| \Wv_{t+1} - \Wv_t \right\|_{Z_t}
&\le
\left\| \eta \nabla \ell_t(\Wv_t) \right\|_{Z_t^{-1}}
\\&~\text{(by $Z_t \succeq \lambda I$)}
\\
&\le
\frac{\eta}{\sqrt{\lambda}} \left\| \nabla \ell_t(\Wv_t) \right\|_2
\\
&=
\frac{\eta}{\sqrt{\lambda}}
\left\| \sig(W_t \x_t) - \y_t \right\|_2 \left\| \x_t \right\|_2
\\
&\le
\frac{2 \eta D}{\sqrt{\lambda}}
\end{align*}
\end{proof}
\end{lemma}

\begin{lemma}[Range term and Negative term of OMD analysis] \label{lem:cr-omd}
For any $t > 0$, we have \[
	\left< \nabla \ellt_t(\Wv_t), \Wv_t - \Wvs \right>
	\le
	\frac{1}{2 \eta} \left\| \Wvs - \Wv_{t-1} \right\|_{Z_{t-1}}^2
	-
	\frac{1}{2 \eta} \left\| \Wvs - \Wv_t \right\|_{Z_{t-1}}^2
	-
	\frac{1}{2 \eta} \left\| \Wv_{t-1} - \Wv_t \right\|_{Z_{t-1}}^2
. \]
\begin{proof}
By the update rule of $W_t$ \[
	W_t
	:=
	\argmin_{W \in \W}
		\ellt(W)
		+
		\frac{1}{2 \eta} \left\| \Wv - \Wv_{t-1} \right\|_{Z_{t-1}}^2
, \] $\ellt(W) = \left< \nabla \ell_t(\Wv_{t-1}), \Wv \right> + \frac{1}{2} \left\| \Wv - \Wv_{t-1} \right\|_{\nabla^2 \ell_t(W_{t-1})}^2$,
we have the first order optimality condition \[
	\left< \eta \nabla \ellt(\Wv_t) + Z_{t-1} \left(\Wv_t - \Wv_{t-1}\right), \Wvs - \Wv_t \right>
	\ge 0
. \] Rearranging the above equation, we have \begin{align*}
\left< \eta \nabla \ellt(\Wv_t), \Wv_t - \Wvs \right>
&\le
\left< Z_{t-1} \left(\Wv_t - \Wv_{t-1}\right), \Wvs - \Wv_t \right>
\\
&=
\frac{1}{2} \left\| \Wvs - \Wv_{t-1} \right\|_{Z_{t-1}}^2
- \frac{1}{2} \left\| \Wvs - \Wv_t \right\|_{Z_{t-1}}^2
- \frac{1}{2} \left\| \Wv_{t-1} - \Wv_t \right\|_{Z_{t-1}}^2
\end{align*}
and complete the proof.
\end{proof}
\end{lemma}

\paragraph{Proof of Lemma \ref{lem:diff-rwd-upper-bound}} \label{proof:lem-diff-rwd-upper-bound}
\begin{proof}
The following steps is by Taylor's Theorem to perform second order expansion.
\begin{align*}
\left| \rwdv^\top \sig(\Ws \x) - \rwdv^\top \sig(W_t \x) \right|
\le&~
\left| \rwdv^\top \nabla \sig(W_t \x) (\Wvs - \Wv_t) \x \right|
\\ &
+ \left| \frac{1}{2} \x^\top (\Wvs - \Wv_t)^\top (D^2 \sig(\xi_1) [\rwdv]) (\Wvs - \Wv_t) \x \right|
\\
=&~
\left| \left(\nabla \sig(W_t \x) \rwdv \otimes \x\right)^\top (\Wvs - \Wv_t) \right|
\\ &
+ \frac{1}{2} \left| (\Wvs - \Wv_t)^\top (D^2 \sig(\xi_1) [\rwdv] \otimes \x \x^\top) (\Wvs - \Wv_t) \right|
\end{align*}
for some $\xi_1$ in the line segment between $W_t \x$ and $\Ws \x$.
The first term can be further expanded as follows,
\begin{align*}
\left| \left(\nabla \sig(W_t \x) \rwdv \otimes \x\right)^\top (\Wvs - \Wv_t) \right|
&= \left| \left((\nabla \sig(W_t \x) \otimes \x) \rwdv \right)^\top (\Wvs - \Wv_t) \right|
\\
&\le \left\| (\nabla \sig(W_t \x) \otimes \x) \rwdv \right\|_{Z_t^{-1}} \left\| \Wvs - \Wv_t \right\|_{Z_t}
\\
&\le \sqrt{\beta_t(\delta)} \left\| (\nabla \sig(W_t \x) \otimes \x) \rwdv \right\|_{Z_t^{-1}}
\end{align*}
The first equality is by the property of Kronecker product, and the second inequality is by Cauchy-Schwarz inequality, and third inequality is by Theorem \ref{thm:confidence-region}.
\par 
The second term can be expended to
\begin{align*}
\left| (\Wvs - \Wv_t)^\top (D^2 \sig(\xi_1) [\rwdv] \otimes \x \x^\top) (\Wvs - \Wv_t) \right|
&\le
\left\| \Wvs - \Wv_t \right\|_{Z_t}
\left\| Z_t^{-\frac{1}{2}} (D^2 \sig(\xi_1) [\rwdv] \otimes \x \x^\top) Z_t^{-\frac{1}{2}} \right\|_{2,2}
\left\| \Wvs - \Wv_t \right\|_{Z_t}
\\
&\le
\beta_t(\delta)
\left\| Z_t^{-\frac{1}{2}} (D^2 \sig(\xi_1) [\rwdv] \otimes \x \x^\top) Z_t^{-\frac{1}{2}} \right\|_{2,2}
\end{align*}
The first inequality is by Cauchy-Schwarz inequality and definition of operator norm,
and the second inequality is by Theorem \ref{thm:confidence-region}.
For the last term
\begin{align*}
\left\| Z_t^{-\frac{1}{2}} (D^2 \sig(\xi_1) [\rwdv] \otimes \x \x^\top) Z_t^{-\frac{1}{2}} \right\|_{2,2}
&= \lambda_{\max}(Z_t^{-\frac{1}{2}} (D^2 \sig(\xi_1) [\rwdv] \otimes \x \x^\top) Z_t^{-\frac{1}{2}})
\\
&\le \lambda_{\max}(Z_t^{-\frac{1}{2}} (\|\rwdv\|_1 I_K \otimes \x \x^\top) Z_t^{-\frac{1}{2}})
\\
&= \|\rwdv\|_1 \lambda_{\max}(Z_t^{-\frac{1}{2}} (I_K \otimes \x \x^\top) Z_t^{-\frac{1}{2}})
\\
&= \|\rwdv\|_1 \lambda_{\max}(Z_t^{-\frac{1}{2}} (I_K \otimes \x) (I_K \otimes \x)^\top Z_t^{-\frac{1}{2}})
\\
&= \|\rwdv\|_1 \lambda_{\max}((I_K \otimes \x)^\top Z_t^{-\frac{1}{2}} Z_t^{-\frac{1}{2}} (I_K \otimes \x))
\\
&= \|\rwdv\|_1 \left\| Z_t^{-\frac{1}{2}} (I_K \otimes \x) \right\|_{2,2}^2
\end{align*}
The first equality is by the property of operator norm $\|A\|_{2,2} = \lambda_{\max}(A)$ for symmetric matrix $A$,
and the second inequality is by \eqref{ineq:second-derivative-sigma-bound}.
\par 
To conclude, we have
\begin{align*}
\left| \rwdv^\top \sig(\Ws \x) - \rwdv^\top \sig(W_t \x) \right|
&\le
\sqrt{\beta_t(\delta)} \left\| (\nabla \sig(W_t \x) \otimes \x) \rwdv \right\|_{Z_t^{-1}}
+
\frac{\|\rwdv\|_1 }{2} \beta_t(\delta) \left\| Z_t^{-\frac{1}{2}} (I \otimes \x) \right\|_{2,2}^2
\end{align*}
\end{proof}

\begin{lemma} \label{lem:fd-self-norm-t}
If $\lambda \ge 2 L D^2$,
then for any $t > 0$, we have \[
	\sum_{s=1}^{t} \tr\left( \nabla^2\ell_s(W_{s+1}) Z_s^{-1} \right)
	=
	\sum_{s=1}^{t} \left\| \sqrt{\nabla \sig(W_{s+1} \x_s)} \otimes \x_s \right\|_{Z_s^{-1}}^2
	\le
	\tail{t}{\lambda} \left(Kd \ln 2 \tail{t}{\lambda} + m \ln (1 + \frac{D^2}{m \lambda}t)\right)
. \]
\begin{proof}
Let $H_t$ be the original version of $Z_t$,\[
	H_t = \lambda I + \sum_{s=1}^{t-1} \nabla \sig(W_{s+1} \x_s) \otimes \x_s \x_s^\top = \lambda I + \sum_{s=1}^{t-1} \nabla^2 \ell_s(W_{s+1})
. \] Therefore
\begin{align*}
\sum_{s=1}^{t} \left\| \sqrt{\sig(W_{s+1} \x_s)} \x_s \right\|_{Z_s^{-1}}^2
= &~
\sum_{s=1}^{t} \tr(\nabla^2 \ell_s(W_{s+1}) Z_s^{-1})
\\ &~\text{(by FD lemma \ref{lem:fd-reg-sketch-relation})}
\\ \le &~
\sum_{s=1}^{t} \tail{t}{\lambda} \tr(\nabla^2 \ell_s(W_{s+1}) H_s^{-1})
\\ &~\text{(by Lemma \ref{lem:fd-self-origin-norm-t})}
\\ \le &~
\tail{t}{\lambda} \left(Kd \ln \tail{t}{\lambda/2} + m \ln (1 + 2\frac{D^2}{m \lambda}t)\right)
\end{align*}
\end{proof}
\end{lemma}

\begin{lemma} \label{lem:fd-self-origin-norm-t}
If $\lambda \ge 2 L D^2$,
then for any $t > 0$, we have \[
	\sum_{s=1}^{t} \tr(\nabla^2 \ell_s(W_{s+1}) H_s^{-1})
	\le
	Kd \ln \tail{t}{\lambda/2} + m \ln (1 + 2\frac{D^2}{m \lambda}t)
. \]
\begin{proof}
Let \[
	H_t = \lambda I + \sum_{s=1}^{t-1} \nabla \sig(W_{s+1} \x_s) \otimes \x_s \x_s^\top = \lambda I + \sum_{s=1}^{t-1} \nabla^2 \ell_s(W_{s+1})
. \] And $M_t := \frac{\lambda}{2} I + \sum_{s=1}^{t} \nabla^2 \ell_s(W_{s+1})$,
and by $\nabla^2 \ell_t(W_{t+1}) \preceq \frac{\lambda}{2} I$,
we have \[
	H_t
	=
	\frac{\lambda}{2} I + \sum_{s=1}^{t-1} \nabla^2 \ell_s(W_{s+1}) + \frac{\lambda}{2} I
	\succeq
	\frac{\lambda}{2} I + \sum_{s=1}^{t-1} \nabla^2 \ell_s(W_{s+1}) + \nabla^2 \ell_{t}(W_{t+1})
	=
	M_t
. \] Therefore
\begin{align*}
\sum_{s=1}^{t} \tr(\nabla^2 \ell_s(W_{s+1}) H_s^{-1})
\le &~
\sum_{s=1}^{t} \tr(\nabla^2 \ell_s(W_{s+1}) M_s^{-1})
\\ = &~
\sum_{s=1}^{t} \tr((M_s - M_{s-1}) M_s^{-1})
\\ &~\text{(by the concavity of $y = \ln |X|$)}
\\ \le &~
\sum_{s=1}^{t} \ln \frac{\left| M_s \right|}{\left| M_{s-1} \right|}
\\ \le &~
\sum_{s=1}^{t} \ln \frac{\left| M_s \right|}{\left| M_{s-1} \right|}
\\ = &~
\ln \frac{\left| M_t \right|}{\left| \frac{\lambda}{2} I \right|}
\end{align*}
By the FD Lemma \ref{lem:fd-det-origin},
we have \[
	\sum_{s=1}^{t} \tr(\nabla^2 \ell_s(W_{s+1}) H_s^{-1})
	\le
	Kd \ln \tail{t}{\lambda/2} + m \ln (1 + 2\frac{D^2}{m \lambda}t)
. \]
\end{proof}
\end{lemma}

\begin{lemma} \label{lem:fd-self-norm-I-xxt-zt}
If $\lambda \ge 2 L D^2$,
then for any $t > 0$, we have \[
	\sum_{s=1}^{t} \left\| I \otimes \x_s \right\|_{Z_s^{-1}}^2
	\le
	\kappa \tail{t}{\lambda} \left(Kd \ln \tail{t}{\lambda/2} + m \ln (1 + 2 \frac{D^2}{m \lambda}t)\right)
. \]
\begin{proof}
Let $H_t$ be the original version of $Z_t$,\[
	H_t = \lambda I + \sum_{s=1}^{t-1} \nabla^2 \ell_s(W_{s+1})
. \] And $M_t := \frac{\lambda}{2} I + \frac{1}{\kappa} \sum_{s=1}^{t} I \otimes \x_s \x_s^\top$,
and by $\frac{1}{\kappa} (I \otimes \x_s \x_s^\top) \preceq L D^2 I \preceq \frac{\lambda}{2} I$,
we have \[
	H_t
	=
	\frac{\lambda}{2} I + \sum_{s=1}^{t-1} \nabla^2 \ell_s(W_{s+1}) + \frac{\lambda}{2} I
	\succeq
	\frac{\lambda}{2} I
	+ \sum_{s=1}^{t-1} \frac{1}{\kappa} (I \otimes \x_s \x_s^\top)
	+ \frac{1}{\kappa} (I \otimes \x_t \x_t^\top)
	=
	M_t
. \] Therefore
\begin{align*}
\sum_{s=1}^{t} \left\| I \otimes \x_s \right\|_{Z_s^{-1}}^2
= &~
\sum_{s=1}^{t} \tr((I \otimes \x_s \x_s^\top) Z_s^{-1})
\\ &~\text{(by FD Lemma \ref{lem:fd-reg-sketch-relation})}
\\ \le &~
\sum_{s=1}^{t} \tail{s}{\lambda} \tr((I \otimes \x_s \x_s^\top) H_s^{-1})
\\ \le &~
\sum_{s=1}^{t} \tail{s}{\lambda} \kappa \tr(\frac{1}{\kappa} (I \otimes \x_s \x_s^\top) M_s^{-1})
\\ = &~
\tail{t}{\lambda} \kappa \sum_{s=1}^{t} \tr((M_s - M_{s-1}) M_s^{-1})
\\ &~\text{(by the concavity of $y = \ln |X|$)}
\\ \le &~
\tail{t}{\lambda} \kappa \sum_{s=1}^{t} \ln \frac{\left| M_s \right|}{\left| M_{s-1} \right|}
\\ \le &~
\tail{t}{\lambda} \kappa \sum_{s=1}^{t} \ln \frac{\left| M_s \right|}{\left| M_{s-1} \right|}
\\ = &~
\tail{t}{\lambda} \kappa \ln \frac{\left| M_t \right|}{\left| \frac{\lambda}{2} I \right|}
\end{align*}
By the FD Lemma \ref{lem:fd-det-origin},
we have \[
	\sum_{s=1}^{t} \left\| I \otimes \x_s \right\|_{Z_s^{-1}}^2
	\le
	\kappa \tail{t}{\lambda} \left(Kd \ln \tail{t}{\lambda/2} + m \ln (1 + 2\frac{D^2}{m \lambda}t)\right)
. \]
\end{proof}
\end{lemma}

\begin{lemma} \label{lem:fd-self-norm-I-xxt-zt+1}
If $\lambda \ge 2 L D^2$,
then for any $t > 0$, we have \[
	\sum_{s=1}^{t} \left\| I \otimes \x_s \right\|_{Z_{s+1}^{-1}}^2
	\le
	\kappa \tail{t}{\lambda} \left(Kd \ln \tail{t}{\lambda} + m \ln (1 + \frac{D^2}{m \lambda}t)\right)
. \]
\begin{proof}
Let $H_t$ be the original version of $Z_t$,\[
	H_t = \lambda I + \sum_{s=1}^{t-1} \nabla^2 \ell_s(W_{s+1})
. \] Therefore
\begin{align*}
\sum_{s=1}^{t} \left\| I \otimes \x_s \right\|_{Z_{s+1}^{-1}}^2
=&~
\sum_{s=1}^{t} \tr((I \otimes \x_s \x_s^\top) Z_{s+1}^{-1})
\\ &~\text{(by FD Lemma \ref{lem:fd-reg-sketch-relation})}
\\ \le &~
\sum_{s=1}^{t} \tail{s}{\lambda} \tr((I \otimes \x_s \x_s^\top) H_{s+1}^{-1})
\\ \le &~
\kappa \sum_{s=1}^{t} \tail{s}{\lambda} \tr((H_{s+1} - H_s) H_{s+1}^{-1})
\\ &~\text{(by the concavity of $y = \ln |X|$)}
\\ \le &~
\kappa \sum_{s=1}^{t} \tail{s}{\lambda} \ln \frac{\left| H_{s+1} \right|}{\left| H_s \right|}
\le
\tail{t}{\lambda} \kappa \sum_{s=1}^{t} \ln \frac{\left| H_{s+1} \right|}{\left| H_s \right|}
\\ = &~
\tail{t}{\lambda} \kappa \ln \frac{\left| H_t \right|}{\left| \lambda I \right|}
\end{align*}
By the FD Lemma \ref{lem:fd-det-origin},
we have \[
	\sum_{s=1}^{t} \left\| I \otimes \x_s \right\|_{Z_{s+1}^{-1}}^2
	\le
	\kappa \tail{t}{\lambda} \left(Kd \ln \tail{t}{\lambda} + m \ln (1 + \frac{D^2}{m \lambda}t)\right)
. \]
\end{proof}
\end{lemma}


\end{document}